\documentclass{article}
\usepackage[numbers]{natbib}
\pdfoutput=1
\usepackage{amsmath,amsfonts,amssymb,amsthm, color}
\usepackage{algorithm,algorithmic}
\usepackage{fullpage}
\usepackage{times}
\usepackage{enumerate}
\usepackage{ltxtable}
\usepackage{hyperref}
\usepackage{nicefrac}
\usepackage{multirow}
\hypersetup{
	colorlinks,%
	citecolor=black,%
	filecolor=black,%
	linkcolor=black,%
	urlcolor=black
}
\newcolumntype{L}{>{\centering\arraybackslash} m{0.05\textwidth}}
\newcolumntype{S}{>{\centering\arraybackslash} m{0.45\textwidth}}

\newcommand{\ignore}[1]{}

\title{On the Universality of Online Mirror Descent}

\author{
Nathan Srebro\\
TTIC\\
\texttt{  nati@ttic.edu  } \\
\and
Karthik Sridharan\\
TTIC\\
\texttt{  karthik@ttic.edu  } \\
\and
Ambuj Tewari\\
University of Texas at Austin\\
\texttt{ambuj@cs.utexas.edu} \\
}

\usepackage{amsmath,amsfonts,amssymb,amsthm}
\usepackage{graphicx}
\usepackage{color}
\usepackage{tikz}
\usetikzlibrary{arrows,shapes}
\usetikzlibrary{patterns}

 \newcommand{\ind}[1]{\ 1\hspace{-2.3mm}{1} _{\{#1\}}}

\newtheorem{theorem}{Theorem}
\newtheorem{lemma}[theorem]{Lemma}

\newtheorem{definition}{Definition}

\newtheorem{corollary}[theorem]{Corollary}

\newtheorem{proposition}[theorem]{Proposition}

\newcommand{\reals}{\ensuremath{\mathbb{R}}}

\newcommand{\R}{\ensuremath{\mathbb{R}}}

\newcommand{\argmin}[1]{\underset{#1}{\mathrm{argmin}} \:}

\newcommand{\E}[1]{{\mathbb{E}\left[{#1}\right]}}

\newcommand{\ip}[2]{{\left\langle{#1},{#2}\right\rangle}}
\newcommand{\inner}[1]{{\left\langle{#1}\right\rangle}}

\newcommand{\norm}[1]{\left\lVert{#1}\right\rVert}

\newcommand{\z}{z}

\newcommand{\w}{\ensuremath{\mathbf{w}}}

\renewcommand{\v}{\ensuremath{\mathbf{v}}}

\newcommand{\W}{\ensuremath{\mathbf{W}}}

\newcommand{\x}{\ensuremath{\mathbf{x}}}

\newcommand{\wopt}{\w^{\star}} % w minimizing regularized gen error

\newcommand{\F}{{F}}

\newcommand{\X}{\mathcal{X}}

\newcommand{\y}{\mathbf{y}}

\newcommand{\breg}[3]{\Delta_{#1}\left( #2\middle| #3 \right)}
\newcommand{\mc}[1]{\mathcal{#1}}
\newcommand{\mbb}[1]{\mathbb{#1}}
\newcommand{\mbf}[1]{\mathbf{#1}}

\renewcommand{\v}{\mathbf{v}}

\newcommand{\B}{\mc{B}}
\newcommand{\Bd}{\mc{B}^{\star}}
\renewcommand{\W}{\mc{W}}
\newcommand{\Wd}{\mc{W}^\star}
\renewcommand{\X}{\mc{X}}
\newcommand{\Val}{\mc{V}}
\newcommand{\Reg}{\mbf{R}}

\newcommand{\Xd}{\mc{X}^\star}
\renewcommand{\z}{\mbf{z}}
\newcommand{\En}{\mbb{E}}

\renewcommand{\F}{\mc{F}}
\newcommand{\Algo}{\mc{A}}

\newcommand{\V}{V}

\newcommand{\Algomd}{\mc{A}_{\mathrm{MD}}}
\newcommand{\MD}{\mathrm{MD}}

\begin{document}

\maketitle

\begin{abstract}
We show that for a general class of convex online learning
problems, Mirror Descent can always achieve a (nearly) optimal regret guarantee. 
\end{abstract}

\section{Introduction}\label{sec:intro}
Mirror Descent is a first-order optimization procedure which
generalizes the classic Gradient Descent procedure to non-Euclidean
geometries by relying on a ``distance generating function'' specific
to the geometry (the squared $\ell_2$-norm in the case of standard
Gradient Descent) \cite{nemirovski1978cesaro,beck2003mirror}.  Mirror
Descent is also applicable, and has been analyzed, in a stochastic
optimization setting \cite{juditsky2009stochastic} and in an online
setting, where it can ensure bounded online regret
\cite{shalev2007convex}.  In fact, many classical online learning
algorithms can be viewed as instantiations or variants of Online
Mirror Descent, generally either with the Euclidean geometry (e.g. the
Perceptron algorithm \cite{Block62} and Online Gradient Descent \cite{Zinkevich03}), or in the simplex ($\ell_1$ geometry), using an entropic distance generating function (Winnow \cite{Littlestone88} and Multiplicative Weights / Online Exponentiated Gradient algorithm \cite{KivinenWa97}).  More recently, the Online Mirror Descent framework
has been applied, with appropriate distance
generating functions derived for a variety of new learning problems 
like multi-task learning and other matrix learning problems \cite{KakShaTew10}, online PCA \cite{WarKuz07} etc.

In this paper, we show that Online Mirror Descent is, in a sense, {\em  universal}.  That is, for any convex online learning problem, of
a  general form (specified in Section \ref{sec:problem}), if the
problem is online learnable, then it is online learnable, with a
nearly optimal regret rate, using Online Mirror Descent, with an
appropriate distance generating function. Since Mirror descent is a first order method and often has simple and computationally efficient update rules, this makes the result especially attractive. Viewing online learning as
a sequentially repeated game, this means that Online Mirror Descent is a near optimal strategy, guaranteeing an outcome very close to the
value of the game.  

In order to show such universality, we first generalize and refine the
standard Mirror Descent analysis to situations where the constraint
set is not the dual of the data domain, obtaining a general upper
bound on the regret of Online Mirror Descent in terms of the existence
of an appropriate uniformly convex distance generating function
(Section \ref{sec:MD}). We then extend the notion of a {\em
  martingale type} of a Banach space to be sensitive to both the
constraint set and the data domain, and building on results of
\cite{SriTew10}, we relate the value of the online learning repeated game to
this generalized notion of martingale type (Section \ref{sec:martingale}).
Finally, again building on and generalizing the work of \cite{Pisier75}, we
show how having appropriate martingale type guarantees the existence
of a good uniformly convex function (Section \ref{sec:typeconvex}),
that in turn establishes the desired nearly-optimal guarantee on
Online Mirror Descent (Section \ref{sec:opt}).  We mainly build on the analysis of \cite{SriTew10}, who related the value of the online game to the notion of martingale type of a Banach space and uniform convexity when the constraint set and data domain are dual to each other. The main
technical advance here is a non-trivial generalization of their
analysis (as well as the Mirror Descent analysis) to the more general
situation where the constraint set and data domain are chosen independently of each other.
In Section \ref{sec:examples} several examples are provided that demostrate the use of our analysis.

Mirror Descent was initially introduced as a first order deterministic
optimization procedure, with an $\ell_p$ constraint and a matching
$\ell_q$ Lipschitz assumption ($1\leq p \leq 2, 1/q+1/p=1$), was shown
to be optimal in terms of the {\em number of exact gradient
  evaluations} \cite{NemirovskiYu78}.  Shalev-Shwartz and Singer later
observed that the online version of Mirror Descent, again with an
$\ell_p$ bound and matching $\ell_q$ Lipschitz assumption ($1\leq p
\leq 2, 1/q+1/p=1$), is also optimal in terms of the worst-case
(adversarial) online regret.  In fact, in such scenarios stochastic
Mirror Descent is also optimal in terms of the number of samples used.
We emphasize that although in most, if not all, settings known to us
these three notions of optimality coincide, here we focus only on the
worst-case online regret.

Sridharan and Tewri \cite{SriTew10} generalized the optimality of
online Mirror Descent (w.r.t.~the worst case online regret) to
scenarios where learner is constrained to a unit ball of an arbitrary
Banach space (not necessarily and $\ell_p$ space) and the objective
functions have sub-gradients that lie in the {\em dual ball} of the
space---for reasons that will become clear shortly, we refer to this
as the {\em data domain}.  However, often we encounter problems where
the constraint set and data domain are not dual balls, but rather are
arbitrary convex subsets.  In this paper, we explore this more
general, ``non-dual'', variant, and show that also in such scenarios
online Mirror Descent is (nearly) optimal in terms of the (asymptotic)
worst-case online regret.

\section{Online Convex Learning Problem}\label{sec:problem}
An online convex learning problem can be viewed as a multi-round
repeated game where on round $t$, the learner first picks a vector
(predictor) $\w_t$ from some fixed set $\W$, which is a closed convex
subset of a vector space $\B$. Next, the adversary picks a convex cost
function $f_t : \W \mapsto \reals$ from a class of convex functions
$\F$. At the end of the round, the learner pays instantaneous cost
$f_t(\w_t)$. We refer to the strategy used by the learner to pick the
$f_t$'s as an {\em online learning algorithm}. More formally, an online
learning algorithm $\Algo$ for the problem is specified by the mapping
$\Algo : \bigcup_{n \in \mathbb{N}} \F^{n-1} \mapsto \W$. The regret
of the algorithm $\Algo$ for a given sequence of cost functions
$f_1,\ldots,f_n$ is given by
$$
\Reg_n(\Algo,f_1,\ldots,f_n) = \frac{1}{n} \sum_{t=1}^n f_t(\Algo(f_{1:{t-1}}))  - \inf_{\w \in \W} \frac{1}{n} \sum_{t=1}^n f_t(\w) \ .
$$
The goal of the learner (or the online learning algorithm), is to minimize the regret for any $n$.

In this paper, we consider cost function classes $\F$ specified by a convex
subset $\X \subset \Bd$ of the dual space $\Bd$.  We consider various
types of classes, where for all of them,
subgradients\footnote{Throughout we commit to a slight abuse of
  notation, with $\nabla f(\w)$ indicating some sub-gradient of $f$ at
  $\w$ and $\nabla f(\w) \in \X$ meaning that at least one of the
  sub-gradients is in $\X$.} of the functions in $\F$ lie inside
$\X$ (we use the notation $\inner{\x,\w}$ to mean applying linear functional $\x \in \Bd$ on $\w \in \B$) :
\begin{align*}
&\F_{\mathrm{Lip}}(\X) = \left\{f : f \textrm{ is convex }\forall \w
  \in \W, \nabla f(\w) \in \X \right\}, &\F_{\mathrm{lin}}(\X) =
\left\{\w \mapsto \inner{\x,\w} : \x \in \X \right\}, \\
& \F_{\mathrm{sup}}(\X) = \left\{\w \mapsto |\inner{\x,\w} - y| : \x \in \X, y \in [-b,b]  \right\}
\end{align*}

The value of the game is then the best possible worst-case regret
guarantee an algorithm can enjoy.  Formally the value is defined as :
\begin{align}\label{eq:value}
\Val_n(\F,\X,\W) = \inf_{\Algo} \sup_{f_{1:n}  \in \F(\X)} \Reg_n(\Algo,f_{1:n})
\end{align}

It is well known that the value of a game for all the above sets $\F$
is the same.  More generally:
\begin{proposition}\label{prop:val}
If for a convex function class $\F$, we have that $\forall f \in \F, \w \in \W,  \nabla f(\w) \in \X$ then,
$$
\Val_n(\F,\X,\W) \le \Val_n(\F_{\mathrm{lin}},\X,\W)
$$
Furthermore,
$
\quad\quad\quad\Val_n(\F_{\mathrm{Lip}},\X,\W) = \Val_n(\F_{\mathrm{sup}},\X,\W) = \Val_n(\F_{\mathrm{lin}},\X,\W)
$
\end{proposition}
That is, the value for any class $\F$ with subgradients in $\W$, which
include all the above classes, is upper bounded by the value of the
class of linear functionals in $\W$, see e.g.~\cite{AbernethyBaRaTe08}.
In particular, this includes the class $\F_{\mathrm{Lip}}$ which is
the class of {\em all} functions with subgradients in $\W$, and thus,
since $\F_{\mathrm{lin}}(\X) \subset \F_{\mathrm{Lip}}(\X)$ we get the
first equality.  The second equality is shown in \cite{RakSriTew10}.

The class $\F_{\mathrm{sup}}(\X)$ corresponds to linear prediction
with an absolute-difference loss, and thus its value is the best
possible guarantee for online supervised learning with this loss.  We
can define more generally a class $\F_{\ell} =
\left\{\ell(\inner{\x,\w},y) : \x \in \X, y \in [-b,b] \right\}$ for
any 1-Lipschitz loss $\ell$, and this class would also be of the
desired type, with its value upper bounded by
$\Val_n(\F_{\mathrm{lin}},\X,\W)$.  In fact, this setting includes
supervised learning fairly generally, including problems such as
multitask learning and matrix completion, where in all cases $\X$
specifies the data domain\footnote{Note that any convex supervised
  learning problem can necessarily be viewed as linear classification
  with some convex constraint $\W$ on the predictors.}.  The equality
in the above proposition can also be essentially extended to most
other commonly occurring convex loss function classes like say the
hinge loss class with some extra constant factors.

Owing to Proposition \ref{prop:val}, we can focus our attention on the
class $\F_{\mathrm{lin}}$ (as the other two will behave similarly),
and so use the shorthand
\begin{align}\label{eq:linval}
\Val_n(\W,\X) := \Val_n(\F_{\mathrm{lin}},\X,\W)
\end{align}
and henceforth the term value without any qualification refers to value of the linear game. Further, for any $p \in [1,2]$ we us also define : 
\begin{align}\label{eq:vp}
V_p := \inf\left\{V\ \middle|\ \forall n \in \mathbb{N}, \Val_n(\W,\X) \le V n^{-\left(1 - \frac{1}{p}\right)}\right\}
\end{align}

Most prior work on online learning and optimization considers the case
when $\W$ is the unit ball of some Banach space, and $\X$ is the unit
ball of the dual space, i.e.~$\W$ and $\X$ are related to each other through
duality.  In this work, however, we analyze the general problem where
$\X \in \Bd$ is not necessarily the dual ball of $\W$.  It will be
convenient for us, however, to relate the notions of a convex set and
a corresponding norm.  The Minkowski functional of a subset $\mc{K}$
of a vector space $\mc{V}$ is defined as $\norm{\v}_{\mc{K}} :=
\inf\left\{ \alpha > 0 : \v \in \alpha \mc{K} \right\}$.  If $\mc{K}$
is convex and centrally symmetric (i.e. $\mc{K} = -\mc{K}$), then
$\norm{\cdot}_{\mc{K}}$ is a semi-norm.  {\bf Throughout this paper, we will
require that $\W$ and $\X$ are convex and centrally symmetric.}
Further, if the set $\mc{K}$ is bounded then $\norm{\cdot}_{\mc{K}}$
is a norm.  Although not strictly required for our results, for
simplicity we will assume $\W$ and $\X$ are are such that
$\norm{\cdot}_{\W}$ and $\norm{\cdot}_{\X}$ (the Minkowski functionals of the sets $\W$ and $\X$) are norms. Even though we do
this for simplicity, we remark that all the results go through for
semi-norms.  We use $\Xd$ and $\Wd$ to represent the dual of balls
$\X$ and $\W$ respectively, i.e.~the unit balls of the dual norms
$\norm{\cdot}^*_{\X}$ and $\norm{\cdot}^*_{\X}$.

%Given a set $\X \subset \Bd$, we consider two possible game scenarios. First one where at each round $t$, the adversary picks losses $f_t$ from set $\F(\X)$. The second scenario is one where at each round $t$, the adversary has no specific restriction per round and can play from entire set $\F(\Bd)$. However after $n$ rounds, the losses are such that $\frac{1}{n} \sum_{t=1}^n \nabla f_t(\w_t) \in \X$ (ie. average of sub-gradients is from set $\X$). The value of a game is defined as regret of the optimal adversary versus optimal learner algorithm. Specifically we consider the value of the game for the two scenarios defined as follows : 
%\begin{align*}
%\uVal_n(\F,\X,\W) = \inf_{\Algo} \sup_{f_{1:n}  \in \F(\X)} \Reg_n(\Algo,f_{1:n})
%\end{align*}
%The value of the game for the second scenario is given by
%\begin{align*}
%\aVal_n(\F,\X,\W) = \inf_{\Algo} \sup_{\underset{\frac{1}{n} \sum_{t=1}^n \nabla f_t(\w_t) \in \X}{f_{1:n} \in \F(\Bd) :} } \Reg_n(\Algo,f_1,\ldots,f_{n})
%\end{align*}

\section{Mirror Descent and Uniform Convexity}\label{sec:MD}

A key tool in the analysis mirror descent is the notion of strong
convexity, or more generally uniform convexity:

 \begin{definition}
   $\Psi:\B \rightarrow \R$ is $q$-uniformly convex w.r.t. $\|\cdot\|$
   in $\W \subset \B$:  
$$
\forall_{\w,\w'\in\W} \forall_{\alpha\in[0,1]} \;\; \Psi\left(\alpha \w+ (1 - \alpha)\w'\right) \le \alpha \Psi(\w) + (1 - \alpha) \Psi(\w')  - \tfrac{\alpha (1 - \alpha)}{q} \norm{\w - \w'}^q
$$
\end{definition}

It is important to emphasize that in the definition above, the norm
$\norm{.}$ and the subset $\W$ need not be related, and we only
require uniform convexity inside $\W$.  This allows us to relate a
norm with a non-matching ``ball''.  To this end define,
\begin{align*}
D_p := \inf\left\{ \left(\sup_{\w \in \W}
    \Psi(\w)\right)^{\frac{p-1}{p}} ~\middle|~ \Psi:\W \mapsto
  \reals^+ \textrm{ is }\tfrac{p}{p-1}\textrm{-uniformly convex
    w.r.t. }\norm{\cdot}_{\X^*} \textrm{ on $\W$}, \Psi(0)=0  \right\}
\end{align*}

Given a function $\Psi$, the Mirror Descent algorithm,
$\Algo_{\mathrm{MD}}$ is given by
\begin{align} \label{eq:update1}
& \w_{t+1} = \argmin{\w \in \W} \breg{\Psi}{\w}{\w_{t}} + \eta \inner{\nabla f_t(\w_t), \w - \w_t}\\
\textrm{or equivalently~~~~~~} & \w'_{t+1} = \nabla \Psi^*\left( \nabla \Psi(\w_t) - \eta \nabla f_t(\w_t) \right),~~ \w_{t+1} = \argmin{\w \in \W} \breg{\Psi}{\w}{\w'_{t+1}}
\end{align}
where $\breg{\Psi}{\w}{\w'} := \Psi(\w) - \Psi(\w') -
\inner{\nabla \Psi(\w'), \w - \w'}$ is the Bregman divergence and $\Psi^*$ is
the convex conjugate of $\Psi$. As an example notice that when $\Psi(\w) = \frac{1}{2}\norm{\w}_2^2$ then we get back the gradient descent algorithm and when $\W$ is the $d$ dimensional simplex and $\Psi(\w) = \sum_{i=1}^d \w_i \log(1/\w_i)$ then we get the multiplicative weights update algorithm.

%\begin{remark}
%In the case that $f_t$'s and $\Psi$ are defined on all of $\B$ and are lipschitz and strongly convex on all of $\B$ and the learner is allowed to pick $\w_t$'s all of $\B$ with the goal still being that regret against best $\w \in \W$ is small, then it is easy to see that mirror descent algorithm without the projection on to $\W$ enjoys the same regret bound. More over the update step for mirror descent algorithm without projection then becomes
%$$
%\w_t = \argmin{\w}\left\{  \eta \ip{\sum_{i=1}^{t-1} \nabla f_i(\w_i)}{\w} + R(\w)\right\}
%$$
%\end{remark}

%\begin{example}[Gradient Descent]
%When $\Psi(\cdot) = \frac{1}{2} \norm{\cdot}^2_2$, $\Psi^*(\cdot) = \frac{1}{2} \norm{\cdot}_2^2$ and $\nabla \Psi(\w) = \w$
%and hence the update step of mirror descent becomes,
%$$
%\w'_{t+1} = \w_t - \eta \nabla f_t(\w_t)
%$$
%\end{example}
%
%\begin{example}[Gradient Descent]
%When $\Psi(\cdot) = \frac{1}{2} \norm{\cdot}_{q}^2$, $\Psi^*(\cdot) = \frac{1}{2} \norm{\cdot}_p^2$ and $\nabla \Psi(\w) = \w$
%and hence the update step of mirror descent becomes,
%$$
%\w'_{t+1} = \w_t - \eta \nabla f_t(\w_t)
%$$
%\end{example}
%

\begin{lemma}\label{lem:md}
  Let $\Psi : \B \mapsto \reals$ be non-negative and $q$-uniformly
  convex w.r.t. norm $\norm{\cdot}_{\X^*}$ on $\W$. For the Mirror Descent
  algorithm with this $\Psi$, using $\w_1 = \argmin{\w \in \W}
  \Psi(\w)$ and $\eta = \left(\tfrac{\sup_{\w \in \W}\Psi(\w)}{ n
      B}\right)^{1/p} $ we can guarantee that for any $f_{1}, \ldots,
  f_n$ s.t. $\frac{1}{n} \sum_{t=1}^n \norm{\nabla f_t}_{\X}^p \le
  1~~~$ (where $p = \tfrac{q}{q-1}$),
$$
\Reg(\Algomd,f_1,\ldots,f_n) \le 2 \left( \frac{\sup_{\w \in \W} \Psi(\w)}{n} \right)^{\frac{1}{q} } \ .
$$
\end{lemma}
\ignore{
An observation we make from the above guarantee is that all we need for the bound to work is that $\frac{1}{n} \sum_{t=1}^n \norm{\nabla f_t}_{\X}^p \le 1$ and don't specifically need each $\nabla f_t$ itself to be in $\X$. This tells us that as long as the adversary's moves are such that the average sub-gradients after $n$ rounds lie in $\X$, mirror descent algorithm still works. }

Note that in our case we have $\nabla f \in\X$, i.e.~$\norm{\nabla
  f}_{\X} \leq 1$, and so certainly $\frac{1}{n} \sum_{t=1}^n
\norm{\nabla f_t}_{\X}^p \le 1$.  Similarly to the value of the game,
for any $p \in [1,2]$, we define:
\begin{align}\label{eq:mdp}
\MD_p := \inf\left\{D : \exists \Psi, \eta \textrm{ s.t. } \forall n \in \mathbb{N}, \sup_{f_{1:n}  \in \F(\X)} \Reg_n(\Algomd,f_{1:n}) \le D  n^{-(1 - \frac{1}{p} )}  \right\}
\end{align}
where the Mirror Descent algorithm in the above definition is run with the
corresponding $\Psi$ and $\eta$. The constant $\MD_p$ is a characterization of the
best guarantee the Mirror Descent algorithm can provide.  Lemma
\ref{lem:md} therefore implies:

\begin{corollary}
$\ V_p \le \MD_p \le 2 D_p$.
\end{corollary}
\begin{proof}
The first inequality is essentially by the definition of $V_p$ and $\MD_p$. The second inequality follows directly from previous lemma.
\end{proof}

The Mirror Descent bound suggests that as long as we can find an appropriate function $\Psi$ that is  uniformly convex w.r.t. $\norm{\cdot}_\X^*$ we can get a diminishing regret guarantee using Mirror Descent. This 
suggests constructing the following function:
\begin{align}\label{eq:construction1}
\tilde{\Psi}_q := \argmin{\substack{\psi : \psi \textrm{ is }
  q\textrm{-uniformly convex}\\ \textrm{w.r.t. } \norm{\cdot}_{\X^*}\textrm{ on }\W \textrm{ and }\psi \ge 0 }} \sup_{\w \in \W} \Psi(\w) \ .
\end{align}

If no $q$-uniformly convex function exists then $\tilde{\Psi}_q =
\infty$ is assumed by default. The above function is in a sense the
best choice for the Mirror Descent bound in \eqref{lem:md}.  The
question then is: when can we find such appropriate functions and what
is the best rate we can guarantee using Mirror Descent?

\section{Martingale Type and Value}\label{sec:martingale} 

In \cite{SriTew10}, it was shown that the concept of the {\em Martingale type} (also sometimes
called the {\em Haar type}) of
a Banach space and optimal rates for online convex optimization
problem, where $\X$ and $\W$ are duals of each other, are closely
related. In this section we extend the classic notion of Martingale
type of a Banach space (see for instance \cite{Pisier75}) to one that
accounts for the pair $(\Wd,\X)$. Before we proceed with the
definitions we would like to introduce a few necessary notations.
First, throughout we shall use $\epsilon \in \{\pm1\}^\mathbb{N}$ to
represent infinite sequence of signs drawn uniformly at random (i.e.
each $\epsilon_i$ has equal probability of being $+1$ or $-1$). Also
throughout $(\x_n)_{n \in \mathbb{N}}$ represents a sequence of
mappings where each $\x_n : \{\pm 1\}^{n-1} \mapsto \Bd$. We shall
commit to the abuse of notation and use $\x_n(\epsilon)$ to represent
$\x_n(\epsilon) = \x_n(\epsilon_1,\ldots,\epsilon_{n-1})$ (i.e. although
we used entire $\epsilon$ as argument, $\x_n$ only depends on first
$n-1$ signs). We are now ready to give the extended definition of
Martingale type (or M-type) of a pair $(\Wd,\X)$.

\begin{definition}
A pair $(\Wd,\X)$ of subsets of a vector space $\Bd$ is said to be of M-type $p$ if there exists a constant $C \ge 1$ such that for all sequence of mappings $(\x_{n})_{n \ge 1}$ where each $\x_n : \{\pm1\}^{n-1} \mapsto \Bd$ and any $\x_0 \in \Bd$ : 
\begin{align} \label{eq:mtype}
\sup_{n} \E{\norm{\x_0 + \sum_{i=1}^n \epsilon_i \x_i(\epsilon)}_{\Wd}^p} \le C^p \left(\norm{\x_0}_{\X}^p + \sum_{n \ge 1} \E{\norm{\x_n(\epsilon)}_{\X}^p} \right)
\end{align}
\end{definition}
The concept is called Martingale type because $(\epsilon_n \x_n(\epsilon))_{n \in \mathbb{N}}$ is a martingale difference sequence and it can be shown that rate of convergence of martingales in Banach spaces is governed by the rate of convergence of martingales of the form $Z_n = \x_0 + \sum_{i=1}^n \epsilon_i \x_i(\epsilon)$ (which are incidentally called Walsh-Paley martingales). We point the reader to \cite{Pisier75,Pisier11} for more details.
Further, for any $p \in [1,2]$ we also define,
{\small
\begin{align*}
C_p := \inf\left\{C\ \middle|\ \forall \x_0 \in \Bd, \forall (\x_n)_{n \in \mathbb{N}},\ \sup_{n }{\small \E{\norm{\x_0 + \sum_{i=1}^n \epsilon_i \x_i(\epsilon)}_{\Wd}^p} \le C^p \left(\norm{\x_0}_{\X}^p + \sum_{n \ge 1} \mathbb{E}\norm{\x_n(\epsilon)}_{\X}^p \right)} \right\}
\end{align*}}
$C_p$ is useful in determining if the pair $(\Wd,\X)$ has Martingale type $p$.

The results of \cite{SriTew10,RakSriTew10} showing that a Martingale
type implies low regret, actually apply also for ``non-matching'' $\W$ and
$\X$ and, in our notation, imply that $V_p \le 2 C_p$.  Specifically we have the following theorem from \cite{SriTew10,RakSriTew10} :

\begin{theorem}\label{thm:cite} \cite{SriTew10,RakSriTew10}
For any $\W \in \B$ and any $\X \in \Bd$ and any $n \ge 1$,
\begin{align*}
\sup_{\x} \E{\norm{\frac{1}{n} \sum_{i=1}^n \epsilon_i \x_i(\epsilon)}_{\Wd}} \le  \Val_n(\W,\X) \le 2 \sup_{\x} \E{\norm{ \frac{1}{n} \sum_{i=1}^n \epsilon_i \x_i(\epsilon)}_{\Wd}}
\end{align*}
where the supremum above is over sequence of mappings $(\x_n)_{n \ge 1}$ where each $\x_n : \{\pm1\}^{n-1} \mapsto \X$.
\end{theorem}

Our main interest here will is in establishing that low regret implies Martingale type.  To do so, we start with the above theorem to relate value of the online convex optimization game to rate of convergence of martingales in the Banach space. We then extend the result of Pisier in \cite{Pisier75} to the ``non-matching'' setting combining it with the above theorem to finally get :

\begin{lemma}\label{lem:mn}
If for some $r \in (1,2]$ there exists a constant $D > 0$ such that for any $n$,
\begin{align*}
\Val_n(\W,\X) \le D n^{-(1 - \frac{1}{r})}
\end{align*}
then for all $p < r$, we can conclude that any $\x_0 \in \Bd$ and any $\Bd$ sequence of mappings $(\x_{n})_{n \ge 1}$ where each $\x_n : \{\pm1\}^{n-1} \mapsto \Bd$  will satisfy :
\begin{align*}
\sup_n \E{\norm{\x_0 + \sum_{i=1}^n \epsilon_i \x_i(\epsilon)}^p_{\Wd}} \le \left( \frac{1104\ D}{(r - p)^2}\right)^p \left(\norm{\x_0}_{\X}^p + \sum_{i \ge 1} \E{\norm{\x_i(\epsilon)}_{\X}^p} \right)
\end{align*} 
That is, the pair $(\W,\X)$ is of martingale type $p$.
\end{lemma}

The following corollary is an easy consequence of the above lemma.
\begin{corollary}
For any $p \in [1,2]$ and any $p' <p$ : 
$
C_{p'} \le  \frac{1104\ \V_p}{(p - p')^2}
$
\end{corollary}

%Owing to this we introduce the following definition of sup type of the Banach space pair $(\Wd,\X)$ as :
%\begin{align}\label{eq:suptype}
%p^\star := \sup\left\{ p :  (\W,\X) \textrm{ is of M-type }p\right\} 
%\end{align}
%The following corollary is a direct consequence of the upper and lower bounds in Corollary \ref{cor:upper} and Lemma \ref{lem:mn}.
%
%\begin{corollary}
%The value of the game is bounded as ($\tilde{\Theta}$ hides sub-polynomial factors):
%\begin{align}
%&\Val(\W,\X) = \tilde{\Theta}\left(n^{-\left(1 - \frac{1}{p^\star} \right)}\right) 
%\end{align}
%\end{corollary}
%Hence we see that $p^\star$ characterizes optimal rate of the learning problem up to polynomial factors.

\section{Uniform Convexity and Martingale Type}\label{sec:typeconvex}

The classical notion of Martingale type plays a central role in the study of
geometry of Banach spaces. In \cite{Pisier75}, it was shown that a
Banach space has Martingale type $p$ (the classical notion) if and only
if uniformly convex functions with certain properties exist on that
space (w.r.t. the norm of that Banach space). In this section, we
extend this result and show how the Martingale type of a pair $(\Wd,\X)$ are related to existence of certain uniformly
convex functions. Specifically, the following theorem shows that the
notion of Martingale type of pair $(\Wd,\X)$ is equivalent to
the existence of a non-negative function that is uniformly convex w.r.t.
the norm $\norm{\cdot}_{\Xd}$ on $\W$.

\begin{lemma}\label{lem:construct22}
If, for some $p \in (1,2]$, there exists a constant $C > 0$, such that 
for all sequences of mappings $(\x_{n})_{n \ge 1}$ where each $\x_n : \{\pm1\}^{n-1} \mapsto \Bd$ and any $\x_0 \in \Bd$:
$$
\sup_{n} \E{\norm{\x_0 + \sum_{i=1}^n \epsilon_i \x_i(\epsilon)}_{\Wd}^p} \le C^p \left(\norm{\x_0}_{\X}^p + \sum_{n \ge 1} \E{\norm{\x_n(\epsilon)}_{\X}^p} \right)
$$
(i.e. $(\Wd,\X)$ has Martingale type $p$), then there exists a convex function $\Psi : \B \mapsto \reals^+$ with $\Psi(0) = 0$, that is $q$-uniformly convex w.r.t. norm $\norm{\cdot}_{\Xd}$ s.t. $\forall \w \in \B$, $ \frac{1}{q} \norm{\w}_{\Xd}^q \le \Psi(\w) \le \frac{C^q}{q} \norm{\w}_{\W}^q $.
\end{lemma}

The following corollary follows directly from the above lemma.
\begin{corollary}
For any $p \in [1,2]$,  $D_p \le C_p$.
\end{corollary}

The proof of Lemma \ref{lem:construct22} goes further and gives a specific uniformly
convex function $\Psi$ satisfying the desired requirement (i.e.~establishing $D_p \leq C_p$) under the
assumptions of the previous lemma:
{\small
\begin{align}\label{eq:construction2}
\Psi^*_q(\x) &:= \sup\left\{\frac{1}{C^p} \sup_{n} \E{\norm{\x + \sum_{i=1}^n \epsilon_i \x_i(\epsilon)}^p_{\Wd}} - \sum_{i\ge 1} \E{\norm{\x_i(\epsilon)}_{\X}^p} \right\} ~~ , ~~
\Psi_q  :=  (\Psi_q^*)^* \ .
\end{align}}where the supremum above is over sequences $(\x_{n})_{n
\in \mathbb{N}}$ and $p = \frac{q}{q-1}$.

%$$
%\sup_{n} \E{\Psi^*\left(\x_0 + \sum_{i=1}^n \epsilon_i \x_i(\epsilon)\right)}\le \frac{1}{p} \left(\norm{\x_0}_{\X}^p + \sum_{n \ge 1} \E{\norm{\x_n(\epsilon)}_{\X}^p} \right)
%$$
%
%\begin{align*}
%\sup_{\w \in \W} \inner{\w,\x_0 + \sum_{i=1}^n \epsilon_i \x_i(\epsilon)} & \le \eta \left( \sup_{\w \in \W} \Psi(\w) + \Psi^*\left(\frac{\x_0}{\eta} + \frac{1}{\eta}\sum_{i=1}^n \epsilon_i \x_i(\epsilon)\right) \right)\\
%& \le \eta D_p^{q} + \eta \Psi^*\left(\frac{\x_0}{\eta} + \frac{1}{\eta}\sum_{i=1}^n \epsilon_i \x_i(\epsilon)\right)
%\end{align*}
%
%$$
%\E{\norm{\x_0 + \sum_{i=1}^n \epsilon_i \x_i(\epsilon)}_{\W^*} } \le \eta D_p^{q} + \frac{1}{p \eta^{p-1}} \left(\norm{\x_0}_{\X}^p + \sum_{n \ge 1} \E{\norm{\x_n(\epsilon)}_{\X}^p} \right)
%$$
%

\section{Optimality of Mirror Descent}\label{sec:opt}
In the Section \ref{sec:MD}, we saw that if we can find an appropriate
uniformly convex function to use in the mirror descent algorithm, we
can guarantee diminishing regret. However the pending question there
was when can we find such a function and what is the rate we can
gaurantee. In Section \ref{sec:martingale} we introduced the extended
notion of Martingale type of a pair $(\Wd,\X)$ and
how it related to the value of the game. Then, in Section
\ref{sec:typeconvex}, we saw how the concept of M-type related to
existence of certain uniformly convex functions.  We can now combine
these results to show that the mirror descent algorithm is a universal
online learning algorithm for convex learning problems.  Specifically
we show that whenever a problem is online learnable, the mirror
descent algorithm can guarantee near optimal rates:

\begin{theorem}
  If for some constant $V > 0$ and some $q \in [2,\infty)$, $
  \Val_n(\W,\X) \le V n^{-\frac{1}{q}}$ for all $n$, then for any $n >
  e^{q-1}$, there exists regularizer function $\Psi$ and step-size
  $\eta$, such that the regret of the mirror descent algorithm using
  $\Psi$ against any $f_1,\ldots,f_n$ chosen by the adversary is
  bounded as:
\begin{align}\label{eq:mdval}
\Reg_n(\Algomd,f_{1:n}) \le\, 6002\, V\, \log^2(n)\ n^{- \frac{1}{q}}
\end{align}
\end{theorem}
\begin{proof}
Combining Mirror descent guarantee in Lemma \ref{lem:md}, Lemma \ref{lem:construct22} and the lower bound in Lemma \ref{lem:mn} with $p = \frac{q}{q-1} - \frac{1}{\log(n)}$ we get the above statement.
\end{proof}

The above Theorem tells us that, with appropriate $\Psi$ and learning
rate $\eta$, mirror descent will obtain regret at most a factor of
$6002 \log(n)$ from the best possible worst-case upper bound.  We would
like to point out that the constant $V$ in the value of the game
appears linearly and there is no other problem or space related hidden
constants in the bound.
%The following corollary follows immediately from the above theorem of definition of $\suptype$ in Equation \eqref{eq:suptype}
%\begin{corollary}
%For any $n > e^{1/(p^{\star}-1)}$, there exists regularizer function $\Psi$ and step-size $\eta$, such that the regret of the  mirror descent algorithm against any $f_1,\ldots,f_n \in \F$ is bounded as:
%$$
%\Reg_n(\Algomd,f_{1:n}) \le\, O\left(\log^2(n)\, n^{-\left(1 - \frac{1}{p^\star}\right)} \right) 
%$$
%\end{corollary}

%\begin{corollary}
%For any $p' <p$ : 
%$
%\MD_{p'} \le D_{p'} \le C_{p'} \le \frac{1104\ V_p}{(p - p')^2}
%$
%\end{corollary}

%\begin{figure}[h]
%\begin{center}
%\includegraphics[scale=0.3]{blkfig.pdf}
%\end{center}
%\caption{Relationship between the various constants}\label{fig:main}
%\end{figure}
The following figure summarizes the relationship between the various constants. The arrow mark from $C_{p'}$ to $C_p$ indicates that for any  $n$, all the quantities are within $\log^2 n$ factor of each other.
\begin{figure}[h]
\begin{center}
\begin{tikzpicture}[node distance=1.5cm, auto,>=latex',
cond/.style={draw, thin, rounded corners, inner sep=1ex, text centered},
cond1/.style={}]
\node[text width=1.6cm, style=cond] (lower) {\small $p' < p,\ C_{p'}$};
\node[text width=0.5cm, style=cond1,right of=lower] (le1) {\huge $ \le$};
\node[text width=1.6cm, style=cond, right of=le1] (value)
{\small $V_p$};
\node[text width=0.5cm, style=cond1,right of=value] (le2) {\huge $\ \le$};\node[text width =1.6cm, style=cond, right of=le2] (MD) {\small
$\MD_p$};
\node[text width=0.5cm, style=cond1,right of=MD] (le3) {\huge $\ \le$};
\node[text width=1.6cm, style=cond, right of=le3] (D) {\small $D_p$};
\node[text width=0.5cm, style=cond1,right of=D] (le4) {\huge $\ \le$};
\node[text width=1.6cm, style=cond, right of=le4] (C) {\small $C_p$};
\node[text width=1.6cm, style=cond1] at (1.8,-0.55) (le5) {\tiny Lemma \ref{lem:mn} };
\node[text width=3cm, style=cond1] at (1.8,-0.78) (le5) {\tiny (extending Pisier's result \cite{Pisier75})  };
\node[text width=1.6cm, style=cond1] at (4.6,-0.55) (le6) {\tiny Definition of $V_p$ };
\node[text width=3cm, style=cond1] at (7.7,-0.78) (le6) {\tiny (Generalized MD guarantee)  };
\node[text width=3cm, style=cond1] at (8.5,-0.55) (le6) {\tiny Lemma \ref{lem:md} };
\node[text width=3cm, style=cond1] at (10.55,-0.55) (le6) {\tiny Construction of $\Psi$, Lemma \ref{lem:construct2} };
\node[text width=3cm, style=cond1] at (10.7,-0.78) (le5) {\tiny (extending Pisier's result \cite{Pisier75})  };
%\node[text width=1.6cm, style=cond1] at (1.8,-0.5) (le5) {\tiny Lemma \ref{lem:mn} };
%\node[text width=1.6cm, style=cond1] at (1.8,-0.5) (le5) {\tiny Lemma \ref{lem:mn} };
%\path (lower) edge[<-, double distance=1pt] (value)
%      (value) edge[<-, double distance=1pt] (MD)
%      (MD) edge[<-, double distance=1pt] (D)
%      (D) edge[<-, double distance=1pt] (C);
\path[<-, draw, double distance=1pt,sloped] (C) -- +(0,0.75) -| (lower);
\end{tikzpicture}
\end{center}
\caption{Relationship between the various constants}\label{fig:main}
\end{figure}
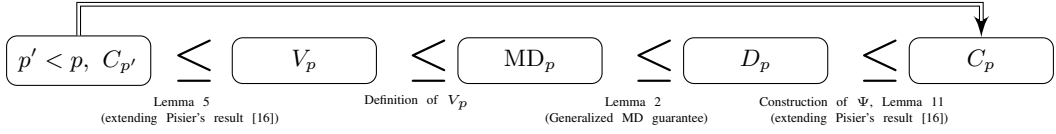

We now provide some general guidelines that will help us in picking out appropriate function $\Psi$ for mirror descent. First we note that though the function $\Psi_q$ in the construction \eqref{eq:construction2} need not be such that $(q \Psi_q(\w))^{1/q}$ is a norm, with a simple modification as noted in \cite{Pisier11} we can make it a norm. This basically tells us that the pair $(\W,\X)$ is online learnable, if and only if we can sandwich a $q$-uniformly convex norm in-between $\Xd$ and a scaled version of $\W$ (for some $q < \infty$). Also note that by definition of uniform convexity, if any function $\Psi$ is $q$-uniformly convex w.r.t. some norm $\norm{\cdot}$ and we have that $\norm{\cdot} \ge c \norm{\cdot}_\X$, then $\tfrac{\Psi(\cdot)}{c^q}$ is $q$-uniformly convex w.r.t. norm $\norm{\cdot}_\X$. These two observations together suggest that given pair $(\W,\X)$ what we need to do is find a norm $\norm{\cdot}$ in between $\norm{\cdot}_\Xd$ and $C \norm{\cdot}_\W$ ($C < \infty$, smaller the $C$ better the bound ) such that $\norm{\cdot}^q$ is $q$-uniformly convex w.r.t $\norm{\cdot}$.

\section{Examples}\label{sec:examples}

We demonstrate our results on several online learning problems,
specified by $\W$ and $\X$.

\paragraph{$\ell_{p}$ non-dual pairs}
It is usual in the literature to consider the case when $\W$ is the unit ball of the $\ell_p$ norm in some finite dimension $d$ while $\X$ is taken to be the unit ball of the dual norm $\ell_q$ where $p,q$ are H\"older conjugate exponents. Using the machinery developed in this paper, it becomes
effortless to consider the non-dual case when $\W$ is the unit ball $B_{p_1}$ of some $\ell_{p_1}$ norm while $\X$ is the unit ball $B_{p_2}$ for {\em arbitrary} $p_1,p_2$ in $[1,\infty]$. We shall use $q_1$ and $q_2$ to represent Holder conjugates of $p_1$ and $p_2$.  Before we proceed we first note that for any $r \in (1,2]$, $\psi_r(\w) := \tfrac{1}{2(r-1)} \|\w\|_r^2$ is $2$-uniformly w.r.t. norm $\norm{\cdot}_r$ (see for instance \cite{Shalev07}).  On the other hand by Clarkson's inequality, we have that for $r \in (2,\infty)$, $\psi_r(\w) := \frac{2^r}{r} \|\w\|_r^r$ is $r$-uniformly convex w.r.t. $\norm{\cdot}_r$.  Putting it together we see that for any $r \in (1,\infty)$, the function $\psi_r$ defined above, is $Q$-uniformly convex w.r.t $\norm{\cdot}_r$ for $Q = \max\{r,2\}$.
The basic technique idea is to be to select $\psi_r$ based on the guidelines in the end of the previous section. 
%To this end, if we define 
%$$
%K_r = \left\{\begin{array}{cl}
%\tfrac{1}{\sqrt{2(r-1)}}  & \textrm{if }r \le 2\\
%2  & \textrm{otherwise}
%\end{array}\right. ~~~~~~~~~~\textrm{and}~~~~~~~~~~\tilde{\psi}_r := d^{Q \max\{\frac{1}{q_2} - \frac{1}{r},0\}} \psi_r
%$$ 
%then $\left(\psi_r(\w) \right)^{1/Q} \le  K_r \norm{\w}_r$. Now recall that
%$\W = B_{p_1}(1)$ and $\Xd = B_{q_2}(1)$ where $q_2 = \tfrac{p_2}{p_2 - 1}$. In general, we have that for any $s,t$ and any $\w \in \reals^d$, $\|\w\|_s \le d^{\max\{\frac{1}{s} - \frac{1}{t} , 0\}} \|\w\|_t$. Owing to the discussion in the end of previous section, since $\psi_r$ is $Q$-uniformly convex w.r.t. $\norm{\cdot}_r$ we have that 
 %then we will have that $\tilde{\psi}_r$ is $Q$-uniformly convex w.r.t $\norm{\cdot}_{q_2}$ and
%$
%\left(\tilde{\psi}_r(\w) \right)^{1/Q} \le K_r d^{\max\{\frac{1}{q_2} - \frac{1}{r},0\} + \max\{\frac{1}{r} - \frac{1}{p_1} ,0\}}\norm{\w}_{p_1}
%$
Finally we show that using $\tilde{\psi}_r := d^{Q \max\{\frac{1}{q_2} - \frac{1}{r},0\}} \psi_r$ in Mirror descent Lemma \ref{lem:md} yields the bound that for any $f_1,\ldots,f_n \in \F$:
$$
\Reg_n(\Algomd,f_{1:n}) \le \frac{2 \max\{2 , \frac{1}{\sqrt{2(r-1)}}\} d^{\max\{\frac{1}{q_2} - \frac{1}{r},0\} + \max\{\frac{1}{r} - \frac{1}{p_1} ,0\}}}{n^{1/\max\{r,2\}}}
$$

%Note that using the above result, dimension free bound is possible only when $\frac{1}{p_1} + \frac{1}{p_2} > 1$. 
%Our basic
%tool is the fact that the function $\Psi(\w) = \tfrac{1}{2(r-1)} \|\w\|_r^2$ is $2$-uniformly
%convex (i.e. strongly convex) w.r.t. $\|\cdot\|_r$ for $r \in (1,2]$. Now, since
%$$ \|\w\|_r \ge d^{-\left(\frac{1}{q_2} - \frac{1}{r} \right)_+} \cdot \|\w\|_{q_2} $$
%for any $q_2 \in [1,\infty]$, $\Psi(\w) = \tfrac{d^{-1(1/q_2 - 1/r)_+}}{2(r-1)} \|\w\|_r^2$ is $2$-uniformly
%convex w.r.t. $q_2$. Now Theorem~\ref{TODO} say that the regret for Mirror Descent in the game with $\W = B_{p_1}(1)$ and
%$\X = B_{p_2}(1)$ is bounded as:
%\[
%	\sqrt{\frac{2}{r-1}} d^{\left(\frac{1}{q_2} - \frac{1}{r}\right)_+} \cdot d^{\left( \frac{1}{r} - \frac{1}{p_1} \right)_+} \cdot \sqrt{n}\ .
%\]
%It is a simple matter to optimize the exponent of $d$ over the choice
%of $r \in (1,2]$ for given $p_1,p_2$. 
The following table summarizes the scenarios where a value of $r=2$,
i.e.~a rate of $D_2/\sqrt{n}$, is possible, and lists the
corresponding values of $D_2$ (up to numeric constant of at most $16$): 
{\footnotesize
\begin{center}
\begin{tabular}{||l||l||l||}
\hline 
$p_1$ Range & $q_2 = \frac{p_2}{p_2 - 1}$ Range &$~~~~~~~~~~~~D_2$\\
\hline
$1 \le p_1 \le 2$ & $ q_2 > 2$ &  $1$ \\ 
$1 \le p_1 \le 2$ & $ p_1 \le q_2 \le 2$ & $\sqrt{p_2 - 1}$\\
$1 \le p_1 \le 2$ & $ 1 \le q_2 < p_1$ & $ d^{1/q_2 - 1/p_1} \sqrt{p_2 - 1}$ \\
$p_1 > 2$ & $q_2 > 2$ &  $d^{( 1/2 - 1/p_1 )} $ \\
 $p_1 > 2$ & $ 1 \le q_2 \le 2$  & $d^{( 1/q_2 - 1/p_1 )} $ \\
 $1 \le p_1 \le 2$ & $q_2 = \infty$ & $\sqrt{\log(d)}$\\
 \hline
\end{tabular}
\end{center}
} Note that the first two rows are dimension free, and so apply also
in infinite-dimensional settings, whereas in the other scenarios,
$D_2$ is finite only when the dimension is finite.  An interesting
phenomena occurs when $d$ is $\infty$, $p_1 > 2$ and $q_2 \ge p_1$. In
this case $D_2 = \infty$ and so one cant expect a rate of
$O(\frac{1}{\sqrt{n}})$. However we have $D_{p_2} < 16$ and so can
still get a rate of $n^{-\frac{1}{q_2}}$.  

Ball et al \cite{BallCaLi94} tightly calculate the constants of strong
convexity of squared $\ell_p$ norms, establishing the tightness of
$D_2$ when $p_1=p_2$.  By extending their constructions it is also
possible to show tightness (up to a factor of 16) for all other values
in the table.  Also, Agarwal et al \cite{AgaBarRavWai11} recently
showed lower bounds on the sample complexity of stochastic
optimization when $p_1 = \infty$ and $p_2$ is arbitrary---their lower
bounds match the last two rows in the table.

\paragraph{Non-dual Schatten norm pairs in finite dimensions}
Exactly the same analysis as above can be carried out for Schatten
$p$-norms, i.e.~when $\W = B_{S(p_1)}$, $\X = B_{S(p_2)}$ are the unit
balls of Schatten $p$-norm (the $p$-norm of the singular values) for matrix of dimensions $d_1 \times d_2$.  We get the same results as in the table
above (as upper bounds on $D_2$), with $d =\min\{d_1,d_2\}$. These results again follow using similar arguments as $\ell_p$ case and tight constants for strong convexity parameters of the Schatten norm from \cite{BallCaLi94}.

\paragraph{Non-dual group norm pairs in finite dimensions}
In applications such as multitask learning, groups norms such as
$\|\w\|_{q,1}$ are often used on matrices $\w \in \reals^{k \times d}$
where $(q,1)$ norm means taking the $\ell_1$-norm of the
$\ell_q$-norms of the columns of $\w$. Popular choices include $q =
2,\infty$. Here, it may be quite unnatural to use the dual norm
$(p,\infty)$ to define the space $\X$ where the data lives. For
instance, we might want to consider $\W = B_{(q,1)}$ and $\X =
B_{(\infty,\infty)} = B_{\infty}$.  In such a case we can calculate
that $D_2(\W,\X)=\Theta(k^{1 - \frac{1}{q}}\sqrt{\log(d)}) $ using $\Psi(\w) = \frac{1}{q + r-2}\norm{\w}^2_{q,r}$ where $r = \frac{\log \ d}{\log \ d - 1}$.
\ignore{
More generally we can consider the
problems where $\W = B_{(p_1,r_1)}(1)$ and $\X = B_{(p_2,r_2)}(1)$.
Their duals are $\W = B_{(q_1,s_1)}(1)$ and $\X = B_{(q_2,s_2)}(1)$
where $q_1,q_2$ are Holder conjugates of $p_1$ and $p_2$ and $s_1,s_2$
are Holder conjugates of $r_1$, $r_2$ respectively. We first note that
for any $W \in \reals^{k \times d}$,
$$
\norm{W}_{p,r} \le d^{\max\{\frac{1}{p} - \frac{1}{p'},0\}} k^{\max\{\frac{1}{r} - \frac{1}{r'},0\}}
$$
Further in \cite{KakShaTew10} for instance it is shown that for $q,s \in (1,2]$,  $\Phi^{q,s}(\cdot) = \frac{1}{2(q+s-2)}\norm{\cdot}_{q,s}^2$ is $2$-uniformly convex w.r.t. $\norm{\cdot}_{q,s}$. Hence picking $q$ and $s$ appropriately and using the above inequality to relate to $\X$ and $\W$ norm one can get bounds for mirror descent algorithm a in the $\ell_p$ case.}

\paragraph{Max Norm}
Max-norm has been proposed as a convex matrix regularizer for
application such as matrix completion \cite{SreRenJaa05}.  In the
online version of the matrix completion problem at each time step one
element of the matrix is revealed, corresponding to $\X$ being the set
of all matrices with a single element being $1$ and the rest $0$.
Since we need $\X$ to be convex we can take the absolute convex hull
of this set and use $\X$ to be the unit element-wise $\ell_1$ ball.
Its dual is $\norm{W}_{\Xd} = \max_{i,j} |W_{i,j}|$.  On the other hand
given a matrix $W$, its max-norm is given by $\norm{W}_{\mathrm{max}}
= \min_{U,V : W = UV^\top} \left(\max_{i} \norm{U_i}_2\right) \left(
  \max_j \norm{V_j}_2 \right)$. The set $\W$ is the unit ball under
the max norm.  As noted in \cite{SreShr05} the max-norm ball is
equivalent, up to a factor two, to the convex hull of all rank one
sign matrices.  Let us now make a more general observation. 

\begin{proposition}
Let $\W = \mathrm{abscvx}(\{\w_1,\ldots,\w_K\})$. The Minkowski norm for this $\W$ is given by 
$$ 
\norm{\w}_\W := \inf_{\alpha_1 ,\ldots, \alpha_K : \w = \sum_{i=1}^K \alpha_i  \w_i}\sum_{i=1}^K |\alpha_i| 
$$  
In this case, for any $q \in (1,2]$, if we define the norm :
$$ 
\norm{\w}_{\W,q} = \inf_{\alpha_1 ,\ldots, \alpha_K : \w  = \sum_{i=1}^K \alpha_i \w_i}\left( \sum_{i=1}^K |\alpha_i|^q
\right)^{1/q} 
$$ 
then the function  $\Psi(\w) = \frac{1}{2 (q-1)}\norm{\w}_{\W,q}^2$ is $2$-uniformly convex w.r.t. $\norm{\cdot}_{\W,q}$. Further if we use $q =
\frac{\log K}{ \log K - 1}$, then $\sup_{\w \in \W} \sqrt{\Psi(\w)} =
O(\sqrt{\log K})$.
\end{proposition}
    
Proof of the above proposition is similar to proof of strong convexity of $\ell_q$ norms.  For the max norm case as noted before the norm is equivalent to the norm got by the taking the absolute convex hull of 
the set of all rank one sign matrices. Cardinality of this set is of course $2^{N+M}$. Hence using the above proposition and noting that $\Xd$ is the unit ball of $|\cdot|_{\infty}$ we see that $\Psi$ is obviously $2$-uniformly convex w.r.t. $\norm{\cdot}_{\Xd}$ and so we get a regret bound $O\left(\sqrt{\frac{M+N}{n}} \right)$.  This matches the stochastic
(PAC) learning guarantee \cite{SreShr05}, and is the first guarantee we are
aware of for the max norm matrix completion problem in the online setting.

\paragraph{Interpolation Norms }
Another interesting setting is when the set $\W$ is got by interpolating between unit balls of two other norms $\norm{\cdot}_{\W_1}$ and $\norm{\cdot}_{\W_2}$. Specifically one can consider $\W$ to be the unit ball of two such interpolated norms, the first type of interpolation norm is given by,
\begin{align}\label{eq:norminter1}
\norm{\w}_{\W} = \norm{\w}_{\W_1} + \norm{\w}_{\W_2}
\end{align}
The second type of interpolation norm one can consider is given by 
\begin{align}\label{eq:norminter2}
\norm{\w}_{\W} = \inf_{\w_1 + \w_2 = \w}\left(\norm{\w_1}_{\W_1} + \norm{\w_2}_{\W_2}\right)
\end{align}
In learning problems such interpolation norms are often used to induce certain structures or properties into the regularization. For instance one might want sparsity along with grouping effect in the linear predictors for which elastic-net type regularization introduced by Zou and Hastie \cite{ZouHastie05} (this is captured by interpolation of the first type between $\ell_1$ and $\ell_2$ norms). Another example is in matrix completion problems when we would like the predictor matrix to be decomposable into sum of sparse and low rank matrices as done by Chanrdasekaran et. al \cite{ChaSanParWil09} (here one can use the interpolation norm of second type to interpolate between trace norm and element wise $\ell_1$ norm). Another example where interpolation norms of type two are useful are in 
multi-task learning problems (with linear predictors) as done by Jalali et. al \cite{JalRavSanRua10}.  The basic idea is that the matrix of linear predictors can is decomposed into sum of two matrices one with for instance low entry-wise $\ell_1$ norm and other with low $B_{(2,\infty)}$ group norm (group sparsity).

\noindent While in these applications the set $\W$ used is obtained through interpolation norms, it is typically not natural for the set $\X$ to be the dual ball of $\W$ but rather something more suited to the problem at hand. For instance, for the elastic net regularization case, the set $\X$ usually considered are either the vectors with bounded $\ell_\infty$ norm or bounded $\ell_2$. Similarly for the \cite{JalRavSanRua10} case $\X$ could be either matrices with bounded entries or some other natural assumption that suits the problem.

\noindent  It can be shown that in general for any interpolation norm of first type specified in Equation \ref{eq:norminter1}, 
\begin{align}\label{eq:bndinter1}
D_2(\W,\X) \le 2 \min\{D_2(\W_1,\X), D_2(\W_2,\X) \}
\end{align}
Similarly for the interpolation norm of type two one can in general show that,
\begin{align}\label{eq:bndinter2}
D_2(\W,\X) \le \frac{1}{2} \max\{D_2(\W_1,\X), D_2(\W_2,\X) \}
\end{align}
Using the above bounds one can get regret bounds for mirror descent algorithm with appropriate $\Psi$ and step size $\eta$ for specific examples like the ones mentioned.\\

\noindent The bounds given in Equations \eqref{eq:bndinter1} and \eqref{eq:bndinter2} are only upper bounds and it would be interesting to analyze these cases in more detail and also to analyze interpolation between several norms instead of just two.

\section{Conclusion and Discussion}\label{sec:conclusion}

In this paper we showed that for a general class of convex online
learning problems, there always exists a distance generating function
$\Psi$ such that Mirror Descent using this function achieves a near-optimal regret guarantee.  This shows that a fairly simple first-order method, in
which each iteration requires a gradient computation and a prox-map
computation, is sufficient for online learning in a very general
sense.  Of course, the main challenge is deriving distance generating
functions appropriate for specific problems---although we give two
mathematical expressions for such functions, in equations \eqref{eq:construction1} and \eqref{eq:construction2}, neither is particularly tractable in general. In the end of Section \ref{sec:opt} we do give some general guidelines for choosing the right distance generating function. However obtaining a more explicit and simple procedure at least for reasonable Banach spaces is a very interesting question.

Furthermore, for the Mirror Descent procedure to be efficient, the prox-map of the distance generating function must be efficiently computable, which means that even though a Mirror Descent procedure is always theoretically possible, we might in practice choose to use a non-optimal distance generating function, or even a non-MD procedure. Furthermore, we might also find other properties of $\w$ desirable, such as sparsity, which would bias us toward alternative methods
\cite{LangfordLiZh09,DuchiShSiCh08}.  Nevertheless, in most instances that we are aware of, Mirror Descent, or slight variations of it, is truly an optimal procedure and this is formalized and rigorously establish here.

In terms of the generality of the problems we handle, we required that
the constraint set $\W$ be convex, but this seems unavoidable if we
wish to obtain efficient algorithms (at least in general).
Furthermore, we know that in terms of worst-case behavior, both in the
stochastic and in the online setting, for convex cost functions, the
value is unchained when the convex hull of a non-convex constraint set
\cite{RakSriTew10}. The requirement that the data domain $\X$ be
convex is perhaps more restrictive, since even with non-convex data
domain, the objective is still convex.  Such non-convex $\X$ are
certainly relevant in many applications, e.g. when the data is sparse,
or when $\x \in \X$ is an indicator, as in matrix completion problems
and total variation regularization.  In the total variation
regularization problem, $\W$ is the set of all functions on the
interval $[0,1]$ with total variation bounded by $1$ which is in fact
a Banach space. However set $\X$ we consider here is not the entire
dual ball and in fact is neither convex nor symmetric. It only
consists of evaluations of the functions in $\W$ at points on interval
$[0,1]$ and one can consider a supervised learning problem where the
goal is to use the set of all functions with bounded variations to
predict targets which take on values in $[-1,1]$ .  Although the
total-variation problem is not learnable, the matrix completion
problem certainly is of much interest.  In the matrix completion case,
taking the convex hull of $\X$ does not seem to change the
value, %{\bf (Is this true?)}
but we are unaware of neither a guarantee that the value of the game
is unchanged when a non-convex $\X$ is replaced by its convex hull,
nor of an example where the value does change---it would certainly be
useful to understand this issue.  We view the requirement that $\W$
and $\X$ be symmetric around the origin as less restrictive and mostly
a matter of convenience.

We also focused on a specific form of the cost class $\F$, which
beyond the almost unavoidable assumption of convexity, is taken to be
constrained through the cost sub-gradients.  This is general enough for
considering supervised learning with an arbitrary convex loss in a
worst-case setting, as the sub-gradients in this case exactly
correspond to the data points, and so restricting $\F$ through its sub
gradients corresponds to restricting the data domain.  Following
Proposition \ref{prop:val}, any optimality result for $\F_{\mathrm{Lip}}$ also
applies to $\F_{\mathrm{sup}}$, and this statement can also be easily
extended to any other reasonable loss function, including the
hinge-loss, smooth loss functions such as the logistic loss, and even
strongly-convex loss functions such as the squared loss (in this
context, note that a strongly convex scalar function for supervised
learning does {\em not} translate to a strongly convex optimization
problem in the worst case).  Going beyond a worst-case formulation of
supervised learning, one might consider online repeated games with
other constraints on $\F$, such as strong convexity, or even
constraints on $\{f_t\}$ as a sequence, such as requiring low average
error or conditions on the covariance of the data---these are beyond
the scope of the current paper.

Even for the statistical learning setting, online methods along with online to batch conversion are often preferred due to their efficiency especially in high dimensional problems. In fact for $\ell_p$ spaces in the dual case,
using lower bounds on the sample complexity for statistical learning of these problems, one can show that for large dimensional problems, mirror descent is an optimal procedure even for the statistical learning problem. We would like to consider the question of whether Mirror Descent is optimal for stochastic convex optimization, or equivalently convex statistical learning, setting \cite{juditsky2009stochastic,ShalevShSrSr09,SreTew10} in general.  Establishing such universality would have significant implications, as it would indicate that any (convex) problem that is learnable, is
learnable using a one-pass first-order online method (i.e.~a
Stochastic Approximation approach).

%In fact, under such and such
%restrictions, we can relate the such and such type, which controls
%statistical learnability, to the martingale type, which we show here
%controls online learnability, and establish the stochastic
%universality of Mirror Descent.  However there are problems.  (State
%conjecture) is still open and is our ultimate goal in this line of
%research.

%\subsubsection*{References}
\bibliographystyle{plain}
%\bibliography{mdbib}
\bibliography{bib,mdbib,addbib}

\begin{thebibliography}{10}

\bibitem{AbernethyBaRaTe08}
J.~Abernethy, P.~L. Bartlett, A.~Rakhlin, and A.~Tewari.
\newblock Optimal strategies and minimax lower bounds for online convex games.
\newblock In {\em Proceedings of the Nineteenth Annual Conference on
  Computational Learning Theory}, 2008.

\bibitem{AgaBarRavWai11}
Alekh Agarwal, Peter~L. Bartlett, Pradeep Ravikumar, and Martin~J. Wainwright.
\newblock Information-theoretic lower bounds on the oracle complexity of convex
  optimization.

\bibitem{BallCaLi94}
Keith Ball, Eric~A. Carlen, and Elliott~H. Lieb.
\newblock Sharp uniform convexity and smoothness inequalities for trace norms.
\newblock {\em Invent. Math.}, 115:463--482, 1994.

\bibitem{beck2003mirror}
A.~Beck and M.~Teboulle.
\newblock Mirror descent and nonlinear projected subgradient methods for convex
  optimization.
\newblock {\em Operations Research Letters}, 31(3):167--175, 2003.

\bibitem{Block62}
H.~D. Block.
\newblock The perceptron: A model for brain functioning.
\newblock {\em Reviews of Modern Physics}, 34:123--135, 1962.
\newblock Reprinted in "Neurocomputing" by Anderson and Rosenfeld.

\bibitem{ChaSanParWil09}
V.~Chandrasekaran, S.~Sanghavi, P.~Parrilo, and A.~Willsky.
\newblock {Sparse and low-rank matrix decompositions}.
\newblock In {\em IFAC Symposium on System Identification}, 2009.

\bibitem{DuchiShSiCh08}
J.~Duchi, S.~Shalev-Shwartz, Y.~Singer, and T.~Chandra.
\newblock Efficient projections onto the $\ell_1$-ball for learning in high
  dimensions.
\newblock In {\em Proceedings of the 25th International Conference on Machine
  Learning}, 2008.

\bibitem{JalRavSanRua10}
Ali Jalali, Pradeep Ravikumar, Sujay Sanghavi, and Chao Ruan.
\newblock {A Dirty Model for Multi-task Learning}.
\newblock In {\em NIPS}, December 2010.

\bibitem{juditsky2009stochastic}
A.~Juditsky, G.~Lan, A.~Nemirovski, and A.~Shapiro.
\newblock Stochastic approximation approach to stochastic programming.
\newblock {\em SIAM J. Optim}, 19(4):1574--1609, 2009.

\bibitem{KakShaTew10}
Sham~M. Kakade, Shai Shalev-shwartz, and Ambuj Tewari.
\newblock On the duality of strong convexity and strong smoothness: Learning
  applications and matrix regularization, 2010.

\bibitem{KivinenWa97}
J.~Kivinen and M.~Warmuth.
\newblock Exponentiated gradient versus gradient descent for linear predictors.
\newblock {\em Information and Computation}, 132(1):1--64, January 1997.

\bibitem{LangfordLiZh09}
J.~Langford, L.~Li, and T.~Zhang.
\newblock Sparse online learning via truncated gradient.
\newblock In {\em Advances in Neural Information Processing Systems 21}, pages
  905--912, 2009.

\bibitem{Littlestone88}
N.~Littlestone.
\newblock Learning quickly when irrelevant attributes abound: A new
  linear-threshold algorithm.
\newblock {\em Machine Learning}, 2:285--318, 1988.

\bibitem{nemirovski1978cesaro}
A.~Nemirovski and D.~Yudin.
\newblock On cesaro's convergence of the gradient descent method for finding
  saddle points of convex-concave functions.
\newblock {\em Doklady Akademii Nauk SSSR}, 239(4), 1978.

\bibitem{NemirovskiYu78}
A.~Nemirovski and D.~Yudin.
\newblock {\em Problem complexity and method efficiency in optimization}.
\newblock Nauka Publishers, Moscow, 1978.

\bibitem{Pisier75}
G.~Pisier.
\newblock Martingales with values in uniformly convex spaces.
\newblock {\em Israel Journal of Mathematics}, 20(3--4):326--350, 1975.

\bibitem{Pisier11}
G.~Pisier.
\newblock Martingales in banach spaces (in connection with type and cotype).
\newblock {\em Winter School/IHP Graduate Course}, 2011.

\bibitem{RakSriTew10}
A.~Rakhlin, K.~Sridharan, and A.~Tewari.
\newblock Online learning: Random averages, combinatorial parameters, and
  learnability.
\newblock {\em NIPS}, 2010.

\bibitem{ShalevShSrSr09}
S.~Shalev-Shwartz, O.~Shamir, N.~Srebro, and K.~Sridharan.
\newblock Stochastic convex optimization.
\newblock In {\em COLT}, 2009.

\bibitem{shalev2007convex}
S.~Shalev-Shwartz and Y.~Singer.
\newblock Convex repeated games and fenchel duality.
\newblock {\em Advances in Neural Information Processing Systems}, 19:1265,
  2007.

\bibitem{SreRenJaa05}
Nathan Srebro, Jason D.~M. Rennie, and Tommi~S. Jaakola.
\newblock Maximum-margin matrix factorization.
\newblock In {\em Advances in Neural Information Processing Systems 17}, pages
  1329--1336. MIT Press, 2005.

\bibitem{SreShr05}
Nathan Srebro and Adi Shraibman.
\newblock Rank, trace-norm and max-norm.
\newblock In {\em Proceedings of the 18th Annual Conference on Learning
  Theory}, pages 545--560. Springer-Verlag, 2005.

\bibitem{SreTew10}
Nathan Srebro and Ambuj Tewari.
\newblock Stochastic optimization for machine learning.
\newblock In {\em ICML 2010, tutorial}, 2010.

\bibitem{SriTew10}
K.~Sridharan and A.~Tewari.
\newblock Convex games in {B}anach spaces.
\newblock In {\em Proceedings of the 23nd Annual Conference on Learning
  Theory}, 2010.

\bibitem{Shalev07}
S.Shalev-Shwartz.
\newblock {\em Online Learning: Theory, Algorithms, and Applications}.
\newblock PhD thesis, Hebrew University of Jerusalem, 2007.

\bibitem{WarKuz07}
Manfred~K. Warmuth and Dima Kuzmin.
\newblock Randomized online pca algorithms with regret bounds that are
  logarithmic in the dimension, 2007.

\bibitem{Zinkevich03}
M.~Zinkevich.
\newblock Online convex programming and generalized infinitesimal gradient
  ascent.
\newblock In {\em ICML}, 2003.

\bibitem{ZouHastie05}
Hui Zou and Trevor Hastie.
\newblock Regularization and variable selection via the elastic net.
\newblock {\em Journal of the Royal Statistical Society, Series B},
  67:301--320, 2005.

\end{thebibliography}

\appendix
\begin{center}
{\large \bf Appendix}\\
\end{center}

\begin{proof}[Proof of Lemma \ref{lem:md} (generalized MD guarantee)]
Note that for any $\wopt \in \W$,
\begin{align*}
& \eta \left(\sum_{t=1}^n f_t(\w_t) -  \sum_{t=1}^n f_t(\wopt) \right)  \le \sum_{t=1}^n \ip{\eta \nabla f_t(\w_t)}{\w_t - \wopt}\\
& ~~~~~~~~~~~~~~~ = \sum_{t=1}^n\left( \ip{\eta \nabla f_t(\w_t)}{\w_t - \w'_{t+1}} + \ip{\nabla f_t(\w_t)}{\w'_{t+1} - \wopt} \right)\\
& ~~~~~~~~~~~~~~~ = \sum_{t=1}^n\left( \ip{\eta \nabla f_t(\w_t)}{\w_t - \w'_{t+1}} + \ip{\nabla \Psi(\w_{t}) - \nabla \Psi(\w'_{t+1})}{\w'_{t+1} - \wopt} \right)\\
& ~~~~~~~~~~~~~~~ \le \sum_{t=1}^n\left( \norm{\eta \nabla f_t(\w_t)}_{\X} \norm{\w_t - \w'_{t+1}}_{\Xd} + \ip{\nabla \Psi(\w_{t}) - \nabla \Psi(\w'_{t+1})}{\w'_{t+1} - \wopt} \right)\\
& ~~~~~~~~~~~~~~~ \le \sum_{t=1}^n\left( \frac{\eta^p}{p}\norm{\nabla f_t(\w_t)}_{\X}^p + \frac{1}{q} \norm{\w_t - \w'_{t+1}}_{\Xd}^q +  \ip{\nabla \Psi(\w_{t}) - \nabla \Psi(\w'_{t+1})}{\w'_{t+1} - \wopt} \right)
\end{align*}
Using simple manipulation we can show that
\begin{align*}
\ip{\nabla \Psi(\w_{t}) - \nabla \Psi(\w_{t+1})}{\w_{t+1} - \wopt} = \breg{\Psi}{\wopt}{\w_t} - \breg{\Psi}{\wopt}{\w_{t+1}} - \breg{\Psi}{\w_{t+1}}{\w_t}
\end{align*}
where given any $\w,\w' \in \B$, 
$$\breg{\Psi}{\w}{\w'} := \Psi(\w) - \Psi(\w') - \ip{\nabla \Psi(\w')}{\w - \w'}$$
is the Bregman divergence between $\w$ and $\w'$ w.r.t. function $\Psi$.
Hence,
{
\begin{align*}
& \eta \left( \sum_{t=1}^n f_t(\w_t) -  \sum_{t=1}^n f_t(\wopt) \right)  \\
&  ~~~~~~~~~~~~~~~ \le \sum_{t=1}^n\left( \frac{\eta^p}{p}\norm{\nabla f_t(\w_t)}_{\X}^p + \frac{1}{q} \norm{\w_t - \w'_{t+1}}_{\Xd}^q + \ip{\nabla \Psi(\w_{t}) - \nabla \Psi(\w'_{t+1})}{\w'_{t+1} - \wopt} \right)\\
& ~~~~~~~~~~~~~~~ = \sum_{t=1}^n\left( \frac{\eta^p}{p}\norm{\nabla f_t(\w_t)}_{\X}^p + \frac{1}{q} \norm{\w_t - \w'_{t+1}}_{\Xd}^q + \breg{\Psi}{\wopt}{\w_t} - \breg{\Psi}{\wopt}{\w'_{t+1}} - \breg{\Psi}{\w'_{t+1}}{\w_t}  \right) \\
& ~~~~~~~~~~~~~~~ \le \sum_{t=1}^n\left( \frac{\eta^p}{p}\norm{\nabla f_t(\w_t)}_{\X}^p + \frac{1}{q} \norm{\w_t - \w'_{t+1}}_{\Xd}^q + \breg{\Psi}{\wopt}{\w_t} - \breg{\Psi}{\wopt}{\w_{t+1}} - \breg{\Psi}{\w'_{t+1}}{\w_t}  \right) \\
& ~~~~~~~~~~~~~~~ = \sum_{t=1}^n\left( \frac{\eta^p}{p}\norm{\nabla f_t(\w_t)}_{\X}^p + \frac{1}{q} \norm{\w_t - \w'_{t+1}}_{\Xd}^q   - \breg{\Psi}{\w'_{t+1}}{\w_t}  \right) + \breg{\Psi}{\wopt}{\w_1} - \breg{\Psi}{\wopt}{\w_{n+1}}\\
& ~~~~~~~~~~~~~~~ \le \sum_{t=1}^n\left( \frac{\eta^p}{p}\norm{\nabla f_t(\w_t)}_{\X}^p + \frac{1}{q} \norm{\w_t - \w'_{t+1}}_{\Xd}^q   - \breg{\Psi}{\w'_{t+1}}{\w_t}  \right) + \Psi(\wopt) 
\end{align*}
}
Now since $\Psi$ is $q$-uniformly convex w.r.t. $\norm{\cdot}_{\Xd}$,
for any $\w , \w' \in \Bd$, $\breg{\Psi}{\w'}{\w} \ge \frac{1}{q} \norm{\w - \w'}^q_{\Xd}$. Hence we conclude that
\begin{align*}
\sum_{t=1}^n f_t(\w_t) -  \sum_{t=1}^n f_t(\wopt) & \le\frac{\eta^{p-1}}{p}  \sum_{t=1}^n \norm{\nabla f_t(\w_t)}_{\X}^p  + \frac{\Psi(\wopt)}{\eta} \\
& \le\frac{\eta^{p-1} B n}{p}  + \frac{\sup_{\w \in \W} \Psi(\w)}{\eta} \\
& \le\frac{\eta^{p-1} B n}{p}  + \frac{\sup_{\w \in \W} \Psi(\w)}{\eta} 
\end{align*}
Plugging in the value of $\eta = \left(\frac{\sup_{\w \in \W} \Psi(\w)}{n B}\right)^{1/p}$ we get : 
\begin{align*}
\sum_{t=1}^n f_t(\w_t) -  \sum_{t=1}^n f_t(\wopt) & \le 2 \left( \sup_{\w \in \W} \Psi(\w)\right)^{1/q} (B n)^{1/p} 
\end{align*}
dividing throughout by $n$ conclude the proof.
\end{proof}

\begin{lemma}\label{lem:construct2}
Let $1 < p \le 2$ and $C > 0$ be fixed constants, the following statements are equivalent : 
\begin{enumerate}
\item For all sequence of mappings $(\x_{n})_{n \ge 1}$ where each $\x_n : \{\pm1\}^{n-1} \mapsto \Bd$ and any $\x_0 \in \Bd$:
$$
\sup_{n} \E{\norm{\x_0 + \sum_{i=1}^n \epsilon_i \x_i(\epsilon)}_{\Wd}^p} \le C^p \left(\norm{\x_0}_{\X}^p + \sum_{n \ge 1} \E{\norm{\x_n(\epsilon)}_{\X}^p} \right)
$$
\item There exist a non-negative convex function $\Psi$ on $\B$ with $\Phi(0) = 0$, that is $q$-uniformly convex w.r.t. norm $\norm{\cdot}_{\Xd}$ and for any $\w \in \B$, $ \frac{1}{q} \norm{\w}_{\Xd}^q \le \Psi(\w) \le \frac{C^q}{q} \norm{\w}_{\W}^q $.
\end{enumerate}
\end{lemma}
\begin{proof}
For any $\x \in \Bd$ define $\Psi^* : \Bd \mapsto \mbb{R}$ as
$$
\Psi^*(\x ) := \sup\left\{\left(\frac{1}{C^p} \sup_{n} \E{\norm{\x + \sum_{i=1}^n \epsilon_i \x_i(\epsilon)}_{\Wd}^p} - \sum_{i \ge 1} \E{\norm{\x_i(\epsilon)}_\X^p}\right)\right\}
$$
where the supremum is over sequence of mappings $(\x_{n})_{n \ge 1}$ where each $\x_n : \{\pm1\}^{n-1} \mapsto \Bd$ and the sequence is such that, $\underset{n}{\sup}\ \E{\norm{\x + \sum_{i=1}^n \x_i}_{\Wd}^p} < \infty$. Since supremum of convex functions is a convex function, it is easily verified that $\Psi^*(\cdot)$ is convex. Note that by the definition of M-type in Equation \ref{eq:mtype}, we have that for any $\x_0 \in \Bd$, $\Psi^*(\x_0) \le \norm{\x_0}_\X^p$. On the other hand, note that by considering the sequence of constant mappings, $\x_i = 0$ for all $i \ge 1$, we get that for any $\x_0 \in \Bd$,
\begin{align*}
\Psi^*(\x_0) &= \sup\left\{\left(\frac{1}{C^p} \sup_{n} \E{\norm{\x_0 + \sum_{i=1}^n \epsilon_i \x_i(\epsilon)}_{\Wd}^p} - \sum_{i \ge 1} \E{\norm{\x_i(\epsilon)}_\X^p}\right)\right\} \ge \frac{1}{C^p} \norm{\x_0}_{\Wd}^p
\end{align*}
Thus we can conclude that for any $\x \in \Bd$, $\frac{1}{C^p} \norm{\x}_{\Wd}^p \le \Psi^*(\x) \le \norm{\x}_{\X}^p $. \\

\noindent For any $\x_0 , \y_0 \in \Bd$, by definition of $\Psi^*(\x_0)$ and $\Psi^*(\y_0)$, for any $\gamma > 0$, there exist sequences $(\x_n)_{n \ge 1}$ and $(\y_n)_{n \ge 1}$ s.t. :
\begin{align*}
\Psi^*(\x_0) &\le \left(\frac{1}{C^p} \sup_{n} \E{\norm{\x_0 + \sum_{i =1}^n \epsilon_i \x_{i}(\epsilon)}_{\Wd}^p} - \sum_{i \ge 1} \E{\norm{\x_i(\epsilon)}_\X^p}\right) + \gamma
\end{align*}
and 
\begin{align*}
\Psi^*(\y_0^{(j)}) &\le \left(\frac{1}{C^p} \sup_{n} \E{\norm{\y_0 + \sum_{i=1}^n \epsilon_i \y_i(\epsilon) }_{\Wd}^p} - \sum_{i \ge 1} \E{\norm{\y_i(\epsilon)}_\X^p}\right) + \gamma
\end{align*}
In fact in the above two inequalities if the supremum over $n$ were achieved at some finite $n_0$, by replacing the original sequence by one which is identical up to $n_0$ and for any $i > n_0$ using $\x_i(\epsilon) = 0$ (and similarly $\y_i(\epsilon) = 0$), we can in fact conclude that using these $\x$'s and $\y$'s instead,
\begin{align}\label{eq:gx}
\Psi^*(\x_0) &\le \left(\frac{1}{C^p} \E{\norm{\x_0 + \sum_{i \ge 1} \epsilon_i \x_{i}(\epsilon)}_{\Wd}^p} - \sum_{i \ge 1} \E{\norm{\x_i(\epsilon)}_\X^p}\right) + \gamma
\end{align}
and 
\begin{align}\label{eq:gy}
\Psi^*(\y_0^{(j)}) &\le \left(\frac{1}{C^p} \E{\norm{\y_0 + \sum_{i\ge 1} \epsilon_i \y_i(\epsilon) }_{\Wd}^p} - \sum_{i \ge 1} \E{\norm{\y_i(\epsilon)}_\X^p}\right) + \gamma
\end{align}

Now consider a sequence formed by taking $\z_0 = \frac{\x_0 + \y_0}{2}$ and further let 
$$
\z_1 = \left(\frac{1 + \epsilon_0}{2}\right) \frac{\x_0 - \y_0}{2}  + \left(\frac{1 - \epsilon_0}{2}\right) \frac{\y_0 - \x_0}{2} = \epsilon_0 (\x_0 - \y_0)
$$
and for any $i \ge 2$, define 
$$
\z_i = \left(\frac{1 + \epsilon_0}{2}\right) \epsilon_{i-1} \x_{i-1}(\epsilon) + \left(\frac{1 - \epsilon_0}{2}\right) \epsilon_{i-1} \y_{i-1}(\epsilon)
$$
where $\epsilon_0 \in \{\pm 1\}$ is drawn uniformly at random. That is essentially at time $i=1$ we flip a coin and decide to go with the sequence $(\x_n)_{n \ge 0}$ with probability $1/2$ and $(\y_n)_{n \ge 0}$ with probability $1/2$. Clearly using the sequence $(\z_{n})_{n \ge 1}$, we have that,
\begin{align*}
& \Psi^*\left(\frac{\x_0 + \y_0}{2}\right)  = \sup_{(\z)_{n \ge 1}}\left\{\left(\frac{1}{C^p} \sup_{n} \E{\norm{\frac{\x_0 + \y_0}{2} + \sum_{i=1}^n  \z_i(\epsilon_0,\epsilon)}_{\Wd}^p} - \sum_{i \ge 1} \E{\norm{\z_i(\epsilon_0,\epsilon)}_\X^p}\right)^{1/p} \right\}^p \\
& ~~~~~~~~  \ge \frac{1}{C^p} \E{\norm{\z_0 + \sum_{i\ge 1} \z_i(\epsilon_0,\epsilon) }_{\Wd}^p} - \sum_{i \ge 1} \E{\norm{\z_i(\epsilon_0,\epsilon)}_\X^p}\\
& ~~~~~~~~  = \frac{1}{C^p} \frac{\E{\norm{\x_0 + \sum_{i\ge 1} \epsilon_i \x_i(\epsilon) }_{\Wd}^p} + \E{\norm{\y_0 + \sum_{i\ge 1} \epsilon_i \y_i(\epsilon) }_{\Wd}^p}}{2} - \sum_{i \ge 1} \E{\norm{\z_i(\epsilon_0,\epsilon)}_\X^p}\\
& ~~~~~~~~  = \frac{1}{C^p} \frac{ \E{\norm{\x_0 + \sum_{i\ge 1} \epsilon_i \x_i(\epsilon) }_{\Wd}^p} +  \E{\norm{\y_0 + \sum_{i\ge1} \epsilon_i \y_i(\epsilon) }_{\Wd}^p}}{2} - \sum_{i \ge 1} \E{\norm{\z_i(\epsilon_0,\epsilon)}_\X^p}\\
& ~~~~~~~~  = \frac{1}{C^p} \frac{ \E{\norm{\x_0 + \sum_{i\ge1} \epsilon_i \x_i(\epsilon) }_{\Wd}^p} +  \E{\norm{\y_0 + \sum_{i\ge1} \epsilon_i \y_i(\epsilon) }_{\Wd}^p}}{2} -\norm{\frac{\x_0 - \y_0}{2}}_{\X}^p - \sum_{i \ge 2} \E{\norm{\z_i(\epsilon_0,\epsilon)}_\X^p}\\
& ~~~~~~~~  = \frac{1}{C^p} \frac{ \E{\norm{\x_0 + \sum_{i\ge1} \epsilon_i \x_i(\epsilon) }_{\Wd}^p} + \E{\norm{\y_0 + \sum_{i\ge1} \epsilon_i \y_i(\epsilon) }_{\Wd}^p}}{2} -\norm{\frac{\x_0 - \y_0}{2}}_{\X}^p \\
& ~~~~~~~~  ~~~~~~~~~~ - \sum_{i \ge 1} \frac{\E{\norm{\x_i(\epsilon)}_\X^p} + \E{\norm{\y_i(\epsilon)}_\X^p}}{2}\\ 
&  ~~~~~~~~  =  \frac{\frac{1}{C^p} \En{\norm{\x_0 + \sum_{i\ge1} \epsilon_i \x_i(\epsilon) }_{\Wd}^p} - \sum_{i \ge 1} \En{\norm{\x_i(\epsilon)}}_\X^p + \frac{1}{C^p}  \En{\norm{\y_0 + \sum_{i\ge1} \epsilon_i \y_i(\epsilon) }_{\Wd}^p} - \sum_{i \ge 1} \En{\norm{\y_i(\epsilon)}_\X^p} }{2} \\
& ~~~~~~~~  ~~~~~~~~~~ -\norm{\frac{\x_0 - \y_0}{2}}_{\X}^p \\
& ~~~~~~~~  \ge  \frac{\Psi^*(\x_0) + \Psi^*(\y_0)}{2}  -\norm{\frac{\x_0 - \y_0}{2}}_{\X}^p - \gamma
\end{align*}
where the last step is obtained by using Equations \ref{eq:gx} and \ref{eq:gy}. Since $\gamma$ was arbitrary taking limit we conclude that for any $\x_0$ and $\y_0$, 
$$
\frac{\Psi^*(\x_0) + \Psi^*(\y_0)}{2}  \le \Psi^*\left(\frac{\x_0 + \y_0}{2}\right)  + \norm{\frac{\x_0 - \y_0}{2}}_{\X}^p
$$
Hence we have shown the existence of a convex function $\Psi^*$ that is $p$-uniformly smooth w.r.t. norm $\norm{\cdot}_\X$ such that $\frac{1}{C^p} \norm{\cdot}_{\Wd}^p \le \Psi^*(\cdot) \le \norm{\cdot}_{\X}^p$. Using convex duality we can conclude that the convex conjugate $\Psi$ of function $\Psi^*$, is $q$-uniformly convex w.r.t. norm $\|\cdot\|_{\Xd}$ and is such that $\norm{\cdot}_{\X}^q \le \Psi(\cdot) \le C^q \norm{\cdot}_{\W}^q$. 
That $2$ implies $1$ can be easily verified using the smoothness property of $\Psi^*$.

\end{proof}

\noindent The following sequence of four lemma's give us the essentials towards proving Lemma \ref{lem:mn}. They use similar techniques as in \cite{Pisier75}.

\begin{lemma}
Let $1 < r \le 2$. If there exists a constant $D > 0$ such that any $\x_0 \in \Bd$ and any sequence of mappings $(\x_{n})_{n \ge 1}$, where $\x_{n} : \{\pm 1\}^{n-1} \mapsto \Bd$ satisfy :
\begin{align*}
\forall n \in \mathbb{N}, ~~~~~  \E{\norm{\x_0 + \sum_{i=1}^n \epsilon_i \x_i(\epsilon)}_{\Wd}} \le D (n + 1)^{1/r} \sup_{0 \le i \le n} \sup_{\epsilon} \norm{\x_i(\epsilon)}_{\X}
\end{align*}
then for all $p < r$ and $\alpha_p = \frac{20 D}{r - p} $ we can conclude that any $\x_0 \in \Bd$ and any sequence of mappings $(\x_{n})_{n \ge 1}$, where $\x_{n} : \{\pm 1\}^{n-1} \mapsto \Bd$ will satisfy :
\begin{align*}
\sup_n \E{\norm{\x_0 + \sum_{i=1}^n \epsilon_i \x_i(\epsilon)}_{\Wd}} \le \alpha_p  \sup_{\epsilon} \left(\sum_{i\ge 0}\norm{\x_i(\epsilon)}_{\X}^p\right)^{1/p} 
\end{align*}  
\end{lemma}
\begin{proof}
To begin with note that in the definition of type, if the supremum over $n$ were achieved at some finite $n_0$, then by replacing the original sequence by one which is identical up to $n_0$ and then on for any $i > n_0$ using $\x_i(\epsilon) = 0$ would only tighten the inequality. Hence it suffices to only consider such sequences. Further to prove the statement we only need to consider finite such sequences (ie. sequences such that there exists some $n$ so that for any $i > n$, $\x_i = 0$) and show that the inequality holds for every such $n$ (every such sequence). 

Restricting ourselves to such finite sequences, we now use the shorthand,\\
$
S = \sup_{\epsilon} \left(\sum_{i=0}^n \norm{\x_i(\epsilon)}_{\X}^p \right)^{1/p}
$.
Now define
\begin{align*}
& I_k(\epsilon) = \left\{i \ge 0 \middle| \tfrac{S}{2^{(k+1)/p}} < \|\x_i(\epsilon)\|_\X  \le \tfrac{S}{2^{k/p}}\right\}~~,\\
& T_0^{(k)}(\epsilon) = \inf\{i \in I_k(\epsilon)\}~~\textrm{and}\\
& \forall m \in \mathbb{N},\ T_m^{(k)}(\epsilon) = \inf\{i > T_{m-1}^{(k)}(\epsilon), i \in I_k(\epsilon)\}
\end{align*}
Note that for any $\epsilon \in \{\pm1\}^{\mathbb{N}}$,
$$
S^p \ge \sum_{i \in I_k(\epsilon)} \|\x_i(\epsilon)\|_\X^p > \tfrac{S^p\ |I_k(\epsilon)|}{2^{(k+1)}} 
$$
and so we get that $\sup_{\epsilon} |I_k(\epsilon)| < 2^{k+1}$. From this we conclude that

\begin{align*}
\E{\norm{\x_0 + \sum_{i=1}^n \x_i(\epsilon)}_{\Wd}} & \le \sum_{k \ge 0}\E{\norm{\sum_{i\in I_k(\epsilon)} \x_i(\epsilon)}_{\Wd}} \\
& = \sum_{k \ge 0}\E{\norm{\sum_{i \ge 0} \x_{T_i^{(k)}(\epsilon)}}_{\Wd}} \\
& \le \sum_{k \ge 0} \left( D\ \sup_{\epsilon}\{ |I_k(\epsilon)|^{1/r}\} \sup_{\epsilon}\{\sup_{i \in I_k(\epsilon)} \norm{\x_i(\epsilon)}_{\X}\} \right) \\
& \le \sum_{k \ge 0} \left( D\ 2^{(k+1)/r} \sup_{\epsilon} \sup_{i \in I_k(\epsilon)} \norm{\x_i(\epsilon)}_{\X,\infty} \right) \\
\intertext{}
& \le \sum_{k \ge 0} \left( D\ 2^{(k+1)/r}\  2^{-k/p} S \right) \\
& = D\ 2^{1/r} \ \sum_{k \ge 0}2^{k (\frac{1}{r} - \frac{1}{p})}\ S\\
& \le  \frac{2 D}{1 - 2^{(\frac{1}{r} - \frac{1}{p})}} S\\
& \le \frac{2 D}{1 - 2^{-(r-p)/4}} S\\
& \le \frac{12 D}{r - p} S\\
& = \alpha_p \sup_{\epsilon}\left(\sum_{i=0}^n \norm{\x_i(\epsilon)}_{\X}^p \right)^{1/p}
\end{align*}
\end{proof}

\begin{lemma}
Let $1 < r \le 2$. If there exists a constant $D > 0$ such that any $\x_0 \in \Bd$ and any sequence of mappings $(\x_{n})_{n \ge 1}$, where $\x_{n} : \{\pm 1\}^{n-1} \mapsto \Bd$ satisfy :
\begin{align*}
\forall n \in \mathbb{N}, ~~~~~  \E{\norm{\x_0 + \sum_{i=1}^n \epsilon_i \x_i(\epsilon)}_{\Wd}} \le D (n + 1)^{1/r} \sup_{0 \le i \le n} \sup_{\epsilon} \norm{\x_i(\epsilon)}_{\X}
\end{align*}
then for any $p < r$, any $\x_0 \in \Bd$ and any mapping $(\x_{n})_{n \ge 1}$, where $\x_{n} : \{\pm 1\}^{n-1} \mapsto \Bd$:
\begin{align*}
\mathbb{P}\left(\sup_{n} \norm{\x_0 + \sum_{i=1}^n \epsilon_i \x_i(\epsilon)}_{\Wd} > c \right)  \le  2 \left( \frac{\alpha_p}{c} \right)^{p/(p+1)} \left(\norm{\x_0}_{\X}^p + \sum_{i \ge 1} \E{\norm{\x_i(\epsilon)}_{\X}^p} \right)^{1/(p+1)}
\end{align*}
%\begin{align*}
%\left(\sup_{c > 0} c^{p/2}\ \mathbb{P}\left(\sup_{n} \norm{\x_0 + \sum_{i=1}^n \epsilon_i \x_i(\epsilon)}_{\Wd} > c \right) \right)^{2/p} \le (1 + \alpha_p)^{2/p} \left( \x_0 + \sum_{i \ge 1} \E{\norm{\x_i(\epsilon)}_{\X}^p} \right)^{1/p}
%\end{align*}
\end{lemma}
\begin{proof}
For any $\x_0 \in \Bd$ and sequence $(\x_n)_{n \ge 1}$ define 
$$
V_n(\epsilon) = \sum_{i=0}^n \norm{\x_i(\epsilon)}_{\X}^p
$$
For appropriate choice of $a > 0$ to be fixed later, define stopping time 
$$
\tau(\epsilon) = \inf\left\{ n \ge 0 \middle| V_{n+1} > a^p \right\}
$$
Now for any $c > 0$  we have,
{\small
\begin{align}\label{eq:part}
\mathbb{P}\left(\sup_n \norm{\x_0 + \sum_{i=1}^n \epsilon_i \x_i(\epsilon)}_{\Wd} \hspace{-0.3cm}> c\right)  & \le \mathbb{P}(\tau(\epsilon) < \infty) + \mathbb{P}\left( \tau(\epsilon) = \infty , \sup_{n} \norm{\sum_{i=0}^n \epsilon_i \x_i(\epsilon)}_{\Wd} \hspace{-0.3cm}> c \right) \notag\\
& \hspace{-0.25in}\le \mathbb{P}(\tau(\epsilon) < \infty) + \mathbb{P}\left( \tau(\epsilon) >0,\ \sup_{n} \norm{ \x_{0} + \sum_{i=1}^{n\wedge \tau(\epsilon)} \epsilon_i \x_i(\epsilon)}_{\Wd} \hspace{-0.3cm}> c \right)
\end{align}
}
As for the first term in the above equation note that 
\begin{align}\label{eq:p1}
\mathbb{P}(\tau(\epsilon) < \infty) = \mathbb{P}(\sup_{n} V_n > a^p) \le  \frac{\norm{\x_0}_{\X}^p + \sum_{i \ge 1} \E{\norm{\x_i(\epsilon)}_{\X}^p}}{a^p}
\end{align}
To consider the second term of Equation \ref{eq:part} we note that $\left(\ind{\tau(\epsilon) > 0} (\x_0 + \sum_{i=1}^{n \wedge \tau(\epsilon)} \epsilon_i \x_i(\epsilon) ) \right)_{n \ge 0}$ is a valid martingale (stopped process) and hence, $\left(\norm{\ind{\tau(\epsilon) > 0}(\x_0 + \sum_{i=1}^{n \wedge \tau(\epsilon)} \epsilon_i \x_i(\epsilon))}_{\Wd}\right)_{n \ge 0}$ is a sub-matingale. Hence by Doob's inequality we conclude that,
\begin{align*}
\mathbb{P}\left(T > 0, \ \sup_n \norm{\x_0 + \sum_{i=1}^{n \wedge \tau(\epsilon)} \epsilon_i \x_i(\epsilon)}_{\Wd} > c\right) & \le \frac{1}{c} \sup_n \E{\norm{\ind{\tau(\epsilon) > 0}\left( \x_0 + \sum_{i=1}^{n \wedge \tau(\epsilon)} \epsilon_i \x_i(\epsilon)\right)}_{\Wd}}
\end{align*}
Applying conclusion of the previous lemma we get that
\begin{align*}
\mathbb{P}\left(T > 0, \ \sup_n \norm{\x_0 + \sum_{i=1}^{n \wedge \tau(\epsilon)} \epsilon_i \x_i(\epsilon)}_{\Wd} > c\right) & \le \frac{\alpha_p}{c} \sup_{\epsilon}\left(\ind{\tau(\epsilon) > 0} \left(\norm{\x_0}^p_\X + \sum_{i=1}^{\tau(\epsilon)} \norm{\x_i(\epsilon)}^p_\X \right) \right)^{1/p}\\
& \le \frac{\alpha_p}{c} (a^p)^{1/p} = \frac{\alpha_p\ a}{c} 
\end{align*}
Plugging the above and Equation \ref{eq:p1} into Equation \ref{eq:part} we conclude that:
\begin{align*}
\mathbb{P}\left(\sup_n \norm{ \x_0 + \sum_{i=1}^n \epsilon_i \x_i(\epsilon)}_{\Wd} > c\right) & \le  \frac{\norm{\x_0}_{\X}^p + \sum_{i \ge 1} \E{\norm{\x_i(\epsilon)}_{\X}^p}}{a^p} + \frac{\alpha_p \ a}{c}
\end{align*}
Using $a = \left(\frac{c}{\alpha_p} \left(\norm{\x_0}_{\X}^p + \sum_{i \ge 1} \E{\norm{\x_i(\epsilon)}_{\X}^p}\right) \right)^{1/(p+1)}$ we conclude that
\begin{align*}
\mathbb{P}\left(\sup_n \norm{\x_0 + \sum_{i=1}^n \epsilon_i \x_i(\epsilon)}_{\Wd} > c\right) & \le  2 \left( \frac{\alpha_p}{c} \right)^{p/(p+1)} \left(\norm{\x_0}_{\X}^p + \sum_{i \ge 1} \E{\norm{\x_i(\epsilon)}_{\X}^p} \right)^{1/(p+1)}
\end{align*}
%Hence we conclude that for any $c \ge \norm{\x_0}_{\X}^p + \sum_{i \ge 1} \E{\norm{\x_i(\epsilon)}_{\X}^p}$,
%\begin{align*}
%\mathbb{P}\left(\sup_n \norm{\x_0 + \sum_{i=1}^n \epsilon_i \x_i(\epsilon)}_{\Wd} > c\right) & \le  \frac{1 + \alpha_p}{c^{p/2}} 
%\end{align*}
%Since for $c < 1$, $\frac{1}{c^{p/2}} > 1$ and because probability is at most $1$ we can in fact conclude that for any $c > 0$,
%\begin{align*}
%\mathbb{P}\left(\sup_n \norm{\x_0 + \sum_{i=1}^n \epsilon_i \x_i(\epsilon)}_{\Wd} > c\right) & \le  \frac{1 + \alpha_p}{c^{p/2}} 
%\end{align*}
%Owing to our assumption that $\norm{\x_0}_{\X}^p + \sum_{i \ge 1} \E{\norm{\x_i(\epsilon)}_{\X}^p} = 1$, we conclude that for any $c > 0$,
%\begin{align*}
%c^{p/2} \mathbb{P}\left(\sup_n \norm{\x_0 + \sum_{i=1}^n \epsilon_i \x_i(\epsilon)}_{\Wd} > c\right) & \le   (1 + \alpha_p) \left( \norm{\x_0}_{\X}^p + \sum_{i \ge 1} \E{\norm{\x_i(\epsilon)}_{\X}^p} \right)^{1/2}
%\end{align*}
This conclude the proof.
\end{proof}

\begin{lemma}
Let $1 < r \le 2$. If there exists a constant $D > 0$ such that any $\x_0 \in \Bd$ and any sequence of mappings $(\x_{n})_{n \ge 1}$, where $\x_{n} : \{\pm 1\}^{n-1} \mapsto \Bd$ satisfy :
\begin{align*}
\forall n \in \mathbb{N}, ~~~~~  \E{\norm{\x_0 + \sum_{i=1}^n \epsilon_i \x_i(\epsilon)}_{\Wd}} \le D (n + 1)^{1/r} \sup_{0 \le i \le n} \sup_{\epsilon} \norm{\x_i(\epsilon)}_{\X}
\end{align*}
then for any $p < r$, any $\x_0 \in \Bd$ and any sequence $(\x_{n})_{n \ge 1}$ satisfies :
\begin{align*}
& \sup_{\lambda > 0} \lambda^p\ \mathbb{P}\left( \sup_{n} \norm{\x_0 + \sum_{i=1}^n \epsilon_i \x_i(\epsilon)}_{\Wd} > \lambda \right) \\
& ~~~~~~~~~~ \le \max\left\{4^{\frac{p+1}{p}}  \alpha_p \left(\norm{\x_0}_{\X}^p + \sum_{i \ge 1} \E{\norm{\x_i(\epsilon)}_{\X}^p} \right)^{\frac{1}{p}}, 2^{2p + 3} \log(2)\   \alpha_p^p\ \left(\norm{\x_0}_{\X}^p + \sum_{i \ge 1} \E{\norm{\x_i(\epsilon)}_{\X}^p} \right) \right\}
\end{align*}
\end{lemma}
\begin{proof}
We shall use Proposition 8.53 of Pisier's notes which is restated below to prove this lemma. To this end consider any $\x_0 \in \Bd$ and any sequence $(\x_i)_{i\ge 1}$. Given an $\epsilon \in \{\pm1\}^{\mathbb{N}}$, for any $j \in [M]$ and $i \in \mathbb{N}$ let $\epsilon^{(j)}_i = \epsilon_{(i-1) M + j}$. Let $\z_0 = \x_0\ M^{-1/p}$ and define the sequence $(\z_i)_{i \ge 1}$ as follows, for any $k \in \mathbb{N}$ given by $k = j + (i-1)M$ where $j \in [M]$ and $i \in \mathbb{N}$,
\begin{align*}
\z_k(\epsilon) =  \x_{i}(\epsilon^{(j)})\  M^{-1/p}
\end{align*}
Clearly, 
\begin{align*}
\norm{\z_0}_\X^p + \sum_{k \ge 1} \E{\norm{\z_k(\epsilon)}_\X^p} &= \norm{\x_0}_\X^p + \frac{1}{M} \sum_{j=1}^M \sum_{k \ge 1} \E{\norm{\x_k(\epsilon^{(j)})}_\X^p} \\
&  = \norm{\x_0}_\X^p + \sum_{i \ge 1} \E{\norm{\x_i(\epsilon)}_\X^p} 
\end{align*}
By previous lemma we get that for any $c > 0$,
\begin{align*}
\mathbb{P}\left(\sup_n \norm{\z_0 + \sum_{i=1}^n \epsilon_i \z_i(\epsilon)}_{\Wd} > c\right) & \le  2 \left( \frac{\alpha_p}{c} \right)^{p/(p+1)} \left(\norm{\z_0}_{\X}^p + \sum_{i \ge 1} \E{\norm{\z_i(\epsilon)}_{\X}^p} \right)^{1/(p+1)}\\
& =  2 \left( \frac{\alpha_p}{c} \right)^{p/(p+1)} \left(\norm{\x_0}_{\X}^p + \sum_{i \ge 1} \E{\norm{\x_i(\epsilon)}_{\X}^p} \right)^{1/(p+1)}
\end{align*}
Note that
\begin{align*}
\sup_n \norm{\z_0 + \sum_{i=1}^n \epsilon_i \z_i(\epsilon)}_{\Wd} & = M^{-1/p} \sup_{j \in [M]} \sup_{n} \norm{\x_0 + \sum_{i=1}^n \epsilon^{(j)}_i \x_i(\epsilon^{(j)})}_{\Wd}
\end{align*}
Hence we conclude that
\begin{align*}
\mathbb{P}\left( \sup_{j \in [M]} M^{-1/p} \sup_{n} \norm{\x_0 + \sum_{i=1}^n \epsilon^{(j)}_i \x_i(\epsilon^{(j)})}_{\Wd} > c\right) & \le   2 \left( \frac{\alpha_p}{c} \right)^{\frac{p}{(p+1)}} \left(\norm{\x_0}_{\X}^p + \sum_{i \ge 1} \E{\norm{\x_i(\epsilon)}_{\X}^p} \right)^{\frac{1}{(p+1)}}
\end{align*}
For any $j \in [M]$, defining $Z^{(j)} = \sup_{n} \norm{\x_0 + \sum_{i=1}^n \epsilon^{(j)}_i \x_i(\epsilon^{(j)})}_{\Wd}$ and using Proposition \ref{prop:revholder} we conclude that for any $c > 0$,
\begin{align*}
&\sup_{\lambda > 0} \lambda^p\ \mathbb{P}\left( \sup_{n} \norm{\x_0 + \sum_{i=1}^n \epsilon_i \x_i(\epsilon)}_{\Wd} > \lambda \right)  \\
& ~~~~~~~~~~ \le \max\left\{c , 2 c^p \log\left(\frac{1}{ 1 -  2 \left( \frac{\alpha_p}{c} \right)^{\frac{p}{(p+1)}} \left(\norm{\x_0}_{\X}^p + \sum_{i \ge 1} \E{\norm{\x_i(\epsilon)}_{\X}^p} \right)^{\frac{1}{(p+1)}}} \right)  \right\}
\end{align*}
Picking 
$$
c =   4^{\frac{p+1}{p}}  \alpha_p \left(\norm{\x_0}_{\X}^p + \sum_{i \ge 1} \E{\norm{\x_i(\epsilon)}_{\X}^p} \right)^{1/p}
$$
we conclude that 
\begin{align*}
& \sup_{\lambda > 0} \lambda^p\ \mathbb{P}\left( \sup_{n} \norm{\x_0 + \sum_{i=1}^n \epsilon_i \x_i(\epsilon)}_{\Wd} > \lambda \right) \\
& ~~~~~~~~~~ \le \max\left\{4^{\frac{p+1}{p}}  \alpha_p \left(\norm{\x_0}_{\X}^p + \sum_{i \ge 1} \E{\norm{\x_i(\epsilon)}_{\X}^p} \right)^{\frac{1}{p}}, 2^{2p + 3} \log(2)\   \alpha_p^p\ \left(\norm{\x_0}_{\X}^p + \sum_{i \ge 1} \E{\norm{\x_i(\epsilon)}_{\X}^p} \right) \right\}
\end{align*}
\end{proof}

\begin{lemma}\label{lem:sub3}
Let $1 < r \le 2$. If there exists a constant $D > 0$ such that any $\x_0 \in \Bd$ and any sequence of mappings $(\x_{n})_{n \ge 1}$, where $\x_{n} : \{\pm 1\}^{n-1} \mapsto \Bd$ satisfy :
\begin{align*}
\forall n \in \mathbb{N}, ~~~~~  \E{\norm{\x_0 + \sum_{i=1}^n \epsilon_i \x_i(\epsilon)}_{\Wd}} \le D (n + 1)^{1/r} \sup_{0 \le i \le n} \sup_{\epsilon} \norm{\x_i(\epsilon)}_{\X}
\end{align*}
then for all $p < r$, we can conclude that any $\x_0 \in \Bd$ and any sequence of mappings $(\x_{n})_{n \ge 1}$ where each $\x_n : \{\pm1\}^{n-1} \mapsto \Bd$  will satisfy :
\begin{align*}
\sup_n \E{\norm{\x_0 + \sum_{i=1}^n \epsilon_i \x_i(\epsilon)}^p_{\Wd}} \le \left( \frac{1104\ D}{(r - p)^2}\right)^p \left(\norm{\x_0}_{\X}^p + \sum_{i \ge 1} \E{\norm{\x_i(\epsilon)}_{\X}^p} \right)
\end{align*} 
That is the pair $(\W,\X)$ is of martingale type $p$.
\end{lemma}
\begin{proof}
Given any $p < r$ pick $r > p' >p$, due to the homogeneity of the statement we need to prove, w.l.o.g. we can assume that
\begin{align*}
\norm{\x_0}_{\X}^{p'} + \sum_{i \ge 1} \E{\norm{\x_i(\epsilon)}_{\X}^{p'}} = 1
\end{align*}
Hence by previous lemma, we can conclude that
\begin{align}
\sup_{\lambda > 0} \lambda^{p'}\  \mathbb{P}\left( \sup_{n} \norm{\x_0 + \sum_{i=1}^n \epsilon_i \x_i(\epsilon)}_{\Wd} > \lambda \right) & \le {p'} 2^{2p' + 3} \log(2)\   \alpha_{p'}^{p'}  \le (32\ \alpha_{p'})^{p'} \label{eq:finsublem}
\end{align}
Hence,
\begin{align*}
\E{\sup_{n} \norm{\x_0 + \sum_{i=1}^n \epsilon_i \x_i(\epsilon)}_{\Wd}^{p}} & \le \inf_{a > 0}\left\{ a^{p'} +  p \int_{a}^{\infty} \lambda^{p - 1} \mathbb{P}\left( \sup_{n} \norm{\x_0 + \sum_{i=1}^n \epsilon_i \x_i(\epsilon)}_{\Wd} > \lambda \right) d \lambda \right\}\\
& \le \inf_{a > 0}\left\{ a^{p} + p  (32\ \alpha_{p'})^{p'} \int_{a}^{\infty} \lambda^{p - 1 - p'}   d \lambda \right\} \\
& \le \inf_{a > 0}\left\{ a^{p} + p  (32\ \alpha_{p'})^{p'}  \left[\frac{\lambda^{p  - p'}}{p - p'}\right]_{a}^{\infty} \right\}\\
& \le \inf_{a > 0}\left\{ a^{p} +  (46\ \alpha_{p'})^{p'}   \frac{a^{p  - p'}}{p' - p} \right\}\\
& = 2 \frac{(46\ \alpha_{p'})^{p}}{(p' - p)^{p/p'}}  \le 2 \frac{(46\ \alpha_{p})^{p}}{(p' - p)^{p/p'}} 
\end{align*}
Since $\norm{\x_0}_{\X}^{p'} + \sum_{i \ge 1} \E{\norm{\x_i(\epsilon)}_{\X}^{p'}} = 1$ and $p' > p$, we can conclude that $\norm{\x_0}_{\X}^{p} + \sum_{i \ge 1} \E{\norm{\x_i(\epsilon)}_{\X}^{p}} \ge 1$ and so
\begin{align*}
\E{\sup_{n} \norm{\x_0 + \sum_{i=1}^n \epsilon_i \x_i(\epsilon)}_{\Wd}^{p}} & \le 2 \frac{(46\ \alpha_{p})^{p}}{(p' - p)^{p/p'}} \left( \norm{\x_0}_{\X}^{p} + \sum_{i \ge 1} \E{\norm{\x_i(\epsilon)}_{\X}^{p}} \right)\\
& \le 2 \frac{(46\ \alpha_{p})^{p}}{(p' - p)} \left( \norm{\x_0}_{\X}^{p} + \sum_{i \ge 1} \E{\norm{\x_i(\epsilon)}_{\X}^{p}} \right)
\end{align*}
Since $p'$ can be chosen arbitrarily close to $r$, taking the limit we can conclude that
\begin{align*}
\E{\sup_{n} \norm{\x_0 + \sum_{i=1}^n \epsilon_i \x_i(\epsilon)}_{\Wd}^{p}} & \le 2 \frac{(46\ \alpha_{p})^{p}}{(r - p)} \left( \norm{\x_0}_{\X}^{p} + \sum_{i \ge 1} \E{\norm{\x_i(\epsilon)}_{\X}^{p}} \right)
\end{align*}
Recalling that $\alpha_p = \frac{12 D}{r-p}$ we conclude that
\begin{align*}
\E{\sup_{n} \norm{\x_0 + \sum_{i=1}^n \epsilon_i \x_i(\epsilon)}_{\Wd}^{p}} & \le \left(\frac{1104\ D}{(r - p)^{(p+1)/p}}\right)^p \left( \norm{\x_0}_{\X}^{p} + \sum_{i \ge 1} \E{\norm{\x_i(\epsilon)}_{\X}^{p}} \right)\\
& \le \left(\frac{1104\ D}{(r - p)^2}\right)^p \left( \norm{\x_0}_{\X}^{p} + \sum_{i \ge 1} \E{\norm{\x_i(\epsilon)}_{\X}^{p}} \right)
\end{align*}
This concludes the proof.
\end{proof}

We restate below a proposition from Pisier's note (in \cite{Pisier11}) 
\begin{proposition}[Proposition 8.53 of \cite{Pisier11}] \label{prop:revholder}
Consider a random variable $Z \ge 0$ and a sequence $Z^{(1)},Z^{(2)},\ldots$ drawn iid from some distribution. For some $0 < p < \infty$, $0 < \delta <1$ and $R > 0$, 
\begin{align*}
\sup_{M \ge 1} \mathbb{P}\left(\sup_{m \le M} M^{-1/p} Z^{(m)} > R \right) \le \delta 
~~~\Longrightarrow ~~~
\sup_{\lambda > 0} \lambda^p\ \mathbb{P}\left(Z > \lambda \right) \le \max\left\{R, 2 R^p \log\left(\frac{1}{1 - \delta} \right) \right\}
\end{align*}
\end{proposition}
\vspace{0.2in}

\begin{proof}[{\bf Proof of Lemma \ref{lem:mn}}]
By Theorem \ref{thm:cite} and our assumption that $\Val_n(\W,\X) \le D n^{-(1 - 1/r)}$, we have that for any sequence $(\x_n)_{n \ge 1}$ such that $\x_n : \{\pm1\}^{n-1} \mapsto \X$ and any $n \ge 1$,
$$
\E{\frac{1}{n} \norm{\sum_{i=1}^n \epsilon_i \x_i(\epsilon)}_{\Wd}}   \le D n^{-\left(1 - \frac{1}{r}\right)} 
$$
Hence we can conclude for any sequence $(\x_n)_{n \ge 1}$  such that $\x_n : \{\pm1\}^{n-1} \mapsto \Bd$ and any $n \ge 1$,
$$
\E{\norm{\sum_{i=1}^n \epsilon_i \x_i(\epsilon)}_{\Wd}} \le D n^{\frac{1}{r}}  \sup_{1 \le i \le n} \sup_{\epsilon} \norm{\x_i(\epsilon)}_{\X}
$$
Hence for any $\x_0 \in \Bd$, we have that
\begin{align*}
\E{\norm{\x_0 + \sum_{i=1}^n \epsilon_i \x_i(\epsilon)}_{\Wd}} & \le \E{\norm{\sum_{i=1}^n \epsilon_i \x_i(\epsilon)}_{\Wd}} + \norm{\x_0}_{\Wd}\\
& \le \E{\norm{\sum_{i=1}^n \epsilon_i \x_i(\epsilon)}_{\Wd}} + D \norm{\x_0}_{\X}\\
& \le D n^{\frac{1}{r}} \sup_{1 \le i \le n} \sup_{\epsilon} \norm{\x_i(\epsilon)}_{\X} + D \norm{\x_0}_{\X}\\
& \le 2 D (n+1)^{\frac{1}{r}} \sup_{0 \le i \le n} \sup_{\epsilon} \norm{\x_i(\epsilon)}_{\X} 
\end{align*}
Now applying Lemma \ref{lem:sub3} completes the proof.
\end{proof}

\end{document}